\documentclass{article}

\PassOptionsToPackage{numbers}{natbib}
\usepackage[final]{neurips_2025}
\usepackage[T1]{fontenc}
\usepackage[utf8]{inputenc}
\usepackage{xspace}
\usepackage{graphicx}
\usepackage{subfig}
\usepackage[utf8]{inputenc} 
\usepackage[T1]{fontenc}    
\usepackage{hyperref}       
\usepackage{url}            
\usepackage{booktabs}       
\usepackage{amsfonts}       
\usepackage{nicefrac}       
\usepackage{microtype} 
\usepackage{subcaption}
\usepackage[usenames,dvipsnames]{xcolor}
\renewcommand{\paragraph}[1]{\vspace{-0.2mm}\noindent{{\bf #1\ \ \ }}}
\usepackage{Defs}
\usepackage{booktabs}
\usepackage{multirow}
\usepackage{graphicx}
\usepackage{adjustbox}
\PassOptionsToPackage{margin=1in}{geometry}
\usepackage{geometry}
\usepackage{xcolor}
\usepackage{wrapfig}
\usepackage{colortbl}
\definecolor{darkgreen}{rgb}{0.0, 0.5, 0.0}
\definecolor{lightgreen}{rgb}{0.7, 1.0, 0.7}
\newcommand{\epara}[1]{}
\usepackage{amsmath}
\usepackage{amssymb}
\usepackage{mathtools}
\usepackage{amsthm}
\usepackage{thmtools}
\usepackage{amsthm}
\usepackage{thmtools}
\usepackage{amssymb}
\usepackage{bm}
\usepackage[shortlabels,inline]{enumitem}
\setlist{nosep}
\usepackage{soul}
\usepackage{adjustbox}
\usepackage{multirow}
\usepackage{bigdelim}
\usepackage{rotating}
\usepackage{wrapfig,booktabs}
\usepackage{hyperref}
\usepackage[thinc]{esdiff}
\usepackage{float}
\usepackage{xspace}
\usepackage{makecell}

\usepackage[skip=.2ex plus.1ex minus.1ex]{caption}
\makeatletter
\g@addto@macro{\normalsize}{%
\setlength{\abovedisplayskip}{2pt plus1pt}%
\setlength{\abovedisplayshortskip}{2pt plus1pt}%
\setlength{\belowdisplayskip}{2pt plus1pt}%
\setlength{\belowdisplayshortskip}{2pt plus1pt}}
\let\c@table\c@figure
\makeatother
 \newcommand{\Ps}{\mathcal{P}_{< \infty}}
 
\newcommand{\lms}{ \{\!\!\{ }
\newcommand{\rms}{ \}\!\!\} }

\definecolor{cobalt}{rgb}{0.0, 0.28, 0.67}

\usepackage[capitalize,noabbrev]{cleveref}

\usepackage[normalem]{ulem}
\usepackage[textsize=scriptsize,linecolor=orange!40,backgroundcolor=yellow!15,bordercolor=orange!40]{todonotes}  
\usepackage[rightComments=false,commentColor=blue,
beginComment=\quad//~,beginLComment=/*~, endLComment=~*/]{algpseudocodex}
\newcommand{\mas}{MAS\xspace}

\newcommand{\inexact}{inexact\xspace}

\title{Monotone and Separable Set Functions: Characterizations and Neural Models}

\author{%
  Soutrik Sarangi$^*$  \\
  IIT Bombay\\
\And
  Yonatan Sverdlov$^*$ \\
  Technion \\
 \And
  Nadav Dym \\
  Technion\\
\And  
  Abir De \\
IIT Bombay
}

\begin{document}
\def\thefootnote{*}\footnotetext{\small Soutrik and Yonatan contributed equally.
Contact emails of the authors: \texttt{soutriksarangi14@gmail.com,  yonatans@campus.technion.ac.il}, \texttt{nadavdym@technion.ac.il, abir@cse.iitb.ac.in}}
\maketitle

\begin{abstract}

Motivated by applications for set containment problems, we consider the following fundamental problem: can we design set-to-vector functions so that the natural partial order on sets is preserved, namely $S\subseteq T \text{ if and only if } F(S)\leq F(T) $. We call functions satisfying this property \emph{Monotone and Separating (MAS)} set functions.  
We establish lower and upper bounds for the vector dimension necessary to obtain MAS functions, as a function of the cardinality of the multisets and the underlying ground set. In the important case of an infinite ground set, we show that MAS functions do not exist, but provide a model called \our\ which provably enjoys a relaxed MAS property we name ``weakly MAS'' and is stable in the sense of Holder continuity. We also show that MAS functions can be used to construct universal models that are monotone by construction and can approximate all monotone set functions. 
Experimentally, we consider a variety of set containment tasks. The experiments show the benefit of using our \our\ model, in comparison with standard set models which do not incorporate set containment as an inductive bias. Our implementation is available in \url{https://github.com/structlearning/MASNET}.
\end{abstract}

\section{Introduction}
\label{sec:Intro}
A multiset $\lms x_1,\ldots,x_n \rms$ is an unordered collection of vectors, where order does not matter (like sets) and repetitions are allowed (unlike sets). In recent years, there has been increased interest in neural networks that can map multisets to vectors, with applications for physical simulations \cite{huang2021geometrybackflowtransformationansatz}, processing point clouds \cite{qi2017pointnet}, and graph neural networks \cite{xu2018how}. 
Another important application of multiset-to-vector maps is set-containment search~\cite{roy2023locality,singh2020explaining,engels2023dessert}. Here the goal is to check whether a given multiset $S$ is (approximately)  a subset of $T$, and this is often carried out by learning a multiset-to-vector mapping $F$, and then checking whether  $F(S)\leq F(T)$ (element-wise), in which case, one deduces that $S$ is a subset of $T$. 
In this paper, we look into this problem from a theoretical perspective.

\paragraph{Monotone and Separable functions}We begin our analysis with some definitions: we say that a function $F$ mapping sets to vectors is \emph{monotone} if  $S \subseteq  T \subseteq V$ implies that $F(S)\leq F(T)$, and we will say that $F$ is \emph{separable} if $F(S)\leq F(T)$ implies that $S \subseteq T$. When  $F$ is simultaneously \emph{monotone and separable}, we call it a MAS function. When $F$ is a MAS function, we can safely test whether $F(S)\leq F(T) $, and this will be fully equivalent to checking whether $S\subseteq T$. 

\paragraph{Motivation for MAS functions}It is natural to ask why we're interested in functions that are MAS instead of just being monotone or just separable. For this, we begin with a couple of real-life applications based on the set-containment task: (I) In recommendation system design, one problem is to recommend item having a particular set of features, $S$. Here, we represent each item as a set of corresponding features $T$, and the problem is to find all sets $T$ such that approximately, $S \subseteq T$. (II) Text entailment where we are given a small query sentence $q$ nd the goal is to find the set of corpus items $c$ from a large corpus $C$, where $q \implies c$ or $c$ entails $q$. Here, $q$ and $c$ are typically represented as sets of contextual embeddings ($S$ for $q$ and $T$ for $c$) and the entailment problem can be cast as the problem of checking if $S \subset T$.

In a neural network setting, such problems require us to design a function $F$ such that the "order" between $F(S)$ and $F(T)$ can serve as a boolean test for whether $S$ is a subset of $T$. For this, $F$ needs to satisfy two conditions: (A) $F(S) \leq F(T)$ implies $S \subseteq T$ (B) $F(S) \not \leq F(T)$ implies $S \not \subseteq T$. Here, we like to emphasize that, for high accuracy, $F$ needs to satisfy both conditions A and B as above. Monotone functions satisfy condition (B) without necessarily satisfying condition (A). As a result, if we predict $S \subseteq T$ based on $F(S) \leq F(T)$ when $F$ is monotone but not separable, it will give large number of false positives. Similarly, if $F$ is separable but not monotone, then $F$ satisfies condition (A) without satisfying condition (B). Thus, using such for checking $S \subseteq T$ results in a large number of false negatives. Hence, both monotonicity and separability are necessary to develop an accurate test for set-containment

\paragraph{Related works}
Monotone set functions have been widely studied in the theory of capacities~\cite{choquet1953capacities}, fuzzy measures~\cite{agahifuzzy,monotonefuzzy}, game theory and economics~\cite{grabisch_set_func_game}, combinatorial auctions~\cite{lehmann2001combinatorial,dobzinski2005approximation,feige2009maximizing}, and learning theory~\cite{balcan2012learning,you2017deep,blum1988training,odonnell2003learning,bshouty1993exact,li2025monotonic}. These works primarily focus on scalar-valued set functions or vector-valued models without structural order constraints.

To the best of our knowledge, jointly monotone and separable multiset (MAS) functions, i.e., multiset-to-vector maps that preserve the partial order induced by multiset containment—have not been previously formalized. Classical results on partially ordered sets, such as order dimension~\cite{dushnik1941partially}, provide relevant mathematical background but do not address learnable or differentiable constructions. Our goal is to characterize the existence of MAS functions and develop models that satisfy MAS constraints by construction for multiset containment tasks.

This work is also related to injective multiset representations. Foundational results on permutation-invariant functions include~\cite{zaheer2018deepsets,xu2018how}, with subsequent work analyzing the latent dimension required for injectivity~\cite{amir_finite,wagstaff2019limitations,wang2023polynomial} and proposing differentiable injective or bi-Lipschitz embeddings~\cite{amir2024fourierslicedwassersteinembeddingmultisets,balan2022permutation,sverdlov}. In this work, we study MAS functions, which impose strictly stronger constraints than injectivity (see Subsection~\ref{subsec:inj}).

 An additional goal of this work is stability: in most learning scenarios, we are looking for $S$, which is only approximately a subset of $T$. Accordingly, we would like to design functions $F$ that are not only MAS but also stable, in the sense that when $S$ is approximately a subset of $T$, then  $F(S)$ is approximately dominated by $F(T)$. 

\paragraph{Summary of our goal}
In summary, in this work, we aim to characterize (\mas) set functions that output finite-dimensional set representations and subsequently design neural networks for such functions.  At a high level, we seek to address the following questions:
\textbf{(1)} What are the conditions for existence of a finite dimensional \mas functions?
\textbf{(2)} What are the permissible relaxations to monotonicity and separability if the existential conditions identified in (1) are not satisfied?
\textbf{(3)} What are the possible neural architectures for these functions? \textbf{(4)} Are these functions stable?
We expect that the design of trainable models with built-in guarantees of monotonicity, separability, and stability will significantly enhance the inductive bias for set containment applications. 

\vspace{-4pt}
\subsection{Main Results}
\vspace{-1mm}
 
We address the goal specified in the previous subsection by providing a detailed characterization and neural architecture for monotone and separable set functions. Our main results are as follows: 

\paragraph{Existential characterization of \mas functions}
Our first objective is to determine whether it is possible to obtain a \mas\ set function with finite output dimension $m$. We begin our analysis in the case where set elements are assumed to be taken from a finite ground set $V$, and show that in this case a \mas function exists if, and only if,  $m \geq |V|$. Next, we add the assumption that the cardinality of the input multisets is bounded by $k$, and provide lower and upper bounds for $m$ in this case. When the ground set is infinite,  we show that  MAS set functions do not exist.

\paragraph{Weakly MAS functions} 
In most applications, the ground set is $V = \mathbb{R}^D$, and our analysis shows that in this case, MAS  functions do not exist in general. To address this, we introduce the notion of weakly MAS functions by extending the set functions to parametric set functions $F(S, w)$, where $w$  is a parameter. Weakly  MAS functions requires that (i) $F(:,w)$ is monotone, for all $w$, and (ii) for any $S \not\subset T$, there exists at least one $w \in \mathcal{W}$ such that $F(:, w)$ separates $S$ and $T$. We explain how, by choosing suitable activations, it is possible to construct simple deep sets \cite{zaheer2018deepsets} models which are weakly MAS, and name the resulting model \our.

\paragraph{Stability}
We define the notion of stability as discussed above, separability notions are Boolean and fail to capture graded distance. We would like to guarantee that $F(T)$ is 'almost' larger than $F(S)$ when $S$ is 'almost' a subset of $T$. To address this, we propose a novel asymmetric set distance and present a Holder separability condition~\cite{davidson2025holderstabilitymultisetgraph} which ensures stability in terms of this asymmetric distance.

\paragraph{Monotone Functions} Finally, we show that  MAS functions can be used to provide universal models for computing set-to-vector functions. Thus, MAS functions may be a useful concept also outside of the set containment setting. 

\paragraph{Experiments}  We provide a number of experiments showing that for set containment problems, using our weakly MAS model leads to stronger results in comparison with standard multiset models like DeepSets and SetTransformer, which do not incorporate MAS considerations as an inductive bias.   

\paragraph{Summary of contributions} Our {main contributions} in this paper are:
\begin{enumerate*}[label=\textbf{(\arabic*)}]
   \item We introduce the novel notion of \emph{monotone and separating} (MAS) multiset  functions.
   \item We discuss lower and upper bounds for the embedding dimension for MAS functions, and suggest a model \our, which is weakly MAS.
   \item We show the stability of \our.
   \item Experimentally, we prove the effectiveness of \our for set containment tasks.
\end{enumerate*}
\vspace{-5pt}
\section{Existential characterization of MAS functions}\label{sec:exist}
\vspace{-4pt}
We begin by stating the notation we will use for our paper.

\paragraph{Notation} We denote by $V$ to be the ground set and  $\Ps(V)$ to be  the  collections of all multisets from $V$ with finite cardinality $S=\lms x_1, \cdots, x_s\rms$  where $x_i\in V$.  We denote the space of all multisets with at most $k$ elements by $\Pvk$. 
We use $[d]$ to denote the set $\set{1, \cdots, d}$. Given two vectors $v, u \in \R^d$, we say $v\le u$  if $v[i] \le u[i]$ for all $i\in [d]$. Also, for an element $v \in V$ and a multiset $S \in \Ps(V)$, we use $c_S(v)$ to denote the number of times $v$ occurs in $S$. We say $S\subseteq T$ if $c_S(v) \le c_T(v)$ for all $v \in V$. 

\paragraph{Proofs}  All proofs for the results stated in the paper are provided in Appendix~\ref{app:proofs}.

In this section, our goal is to understand when an MAS function $F:\Pvk \to \RR^m$ or $F:\Ps(V)\to \RR^m$ exists, and if it does, what is the smallest $m$ for which such a mapping exists.  
We denote this minimal dimension as $\mm(V, k)$, and $\mm(V,\infty) $, respectively. 
\vspace{-6pt}
\subsection{Prologue: Relation to injectivity}
\vspace{-6pt}
\label{subsec:inj}
Our inquiry is related to the notion of injective multiset functions, as every MAS function $F$ is, in particular, injective. Indeed,  suppose for two sets $S, T$ we have $F(S) = F(T)$. Then, we have both $F(S) \leq F(T)$ and $F(T) \leq F(S)$. Hence, we have both $S \subseteq T$ and $T \subseteq S$ by separability, which implies $S = T$. Thus, injectivity follows from separability. 

Injective multiset functions $f:\Pvk \to \RR^m $ are attainable even for $m=1$, providing that $V$ is finite or even infinite \cite{zaheer2018deepsets,xu2018how,amir_finite,amir2024fourierslicedwassersteinembeddingmultisets}. As we will see, attaining the strong condition of MAS multiset functions requires high dimension, and in some cases (infinite $V$), it does not even exist. This can be seen even in the following very simple example:


\begin{example}\label{ex:simple} Assume the ground set consists of only two elements $V=\{0,1\}$, and we only look for multisets of cardinality $\leq k=1$. In this case, there are only three multisets in the space $\mathcal{P}_{\leq 1}(V)$: the empty set, $S=\{0\}$, and $T=\{1\}$. As discussed previously, there exists a multiset function $F:\mathcal{P}_{\leq 1}(V)\to \RR $ which is injective. However, no such function can be separated. This is because neither $S$ nor $T$ is a subset of the other. However, because real numbers are totally ordered, we have either $F(S) \le F(T)$ or $F(T)\le F(S)$, violating the separability condition.

We note that the monotonicity can be satisfied bu using functions of the form $F(S)=\sum_{x\in S} f(x)$ for non-negative $f$. Thus, achieving separation is harder than achieving monotonicity. 
\end{example}

The key reason for the non-existence of a scalar MAS set function is that multisets (or sets) are not totally ordered, but scalars are.
However, this issue does not arise for multidimensional set functions, where vectors, similar to sets, are partially ordered. So, when mapping to vectors, when and how can MAS functions be constructed? We will now discuss this.

\subsection{Existence of \mas\ functions:  finite ground set, unbounded cardinality}
\label{subsec:mas-finite}


We study \mas functions $F: \Pvinf \to \RR^m$, where input multisets may have an arbitrarily large finite cardinality, and the ground set is finite $|V|=n$. In this case, we can show that the  smallest possible dimension $\mm(V,k)$ of a \mas function is exactly $n$:
\begin{restatable}[Dimension of MAS Function]{theorem}{finiteground}
\label{thm:dim_geq_n}
For a finite ground set $V$ of size $n$, there exists a \mas function $F: \Pvinf \to \R^n$. In addition, any \mas function must have a dimension of at least $n$. In other words, $\mm(V, \infty) = n$. 
\end{restatable}
\paragraph{Proof Sketch} 
The construction of a monotone embedding with $n=|V|$ simply uses one-hot encoding. Namely, we identify $V$ with $[n]$.  For every $S\in \Pvinf$, we define
\begin{align}
    \label{eq:onehot}
     F(S) = \sum_{s\in S}e_s\in \mathbb{R}^{n}\textstyle 
\end{align}
where $e_s\in \RR^n$ is the vector with $e_s[s]=1$ and $e_s[j]=0$ for all $j\neq s$. It's clear $F$ satisfies all conditions.
For the second part, assume we have an embedding of dimension $m$, for every output dimension $i \in [m]$, 
there is a "maximal singleton element" $v^*_i \in V$  such that $F(\set{v^*_i})[i] \geq F(\set{v})[i], \forall v \in V$.  The value of the set function $F$ applied to $T=\set{v^*_1,\ldots,v_m^*}$  will dominate any singleton, due to the monotonicity of $F$. However, if $m < |V|$, then we can select a $u\in V \setminus T$ and then $F(\set{u}) \leq F(T)$, which contradicts separability. 



\subsection{Existence of MAS functions: finite ground set, finite cardinality} 
We now consider the case where the ground set $V$ is finite, as before, but now the multisets have bounded cardinality  $k < |V|$. Indeed, in many practical applications, input multiset cardinality is much smaller than $|V|$, even when $|V|$ is large. In this setting, we show that we can get a lower embedding dimension $m$ than in Theorem \ref{thm:dim_geq_n}.  In fact, we show the minimal output dimension $m$ of a \mas function $F: \Pvk \to \mathbb{R}^m$ can scale logarithmically with the size of the ground set $|V|$. However, this comes at the price of an exponential dependence on $k$: 
\begin{restatable}[Upper Bound on $\mm(V,k)$]{theorem}{finitegroundsetrestrictedcardinality}
\label{thm:upper_bound_V}
Let $k<n$ be natural numbers, and let $V$ be a ground set with $|V|=n$.   Then there exists a MAS function $F:\Pvk\to \R^m$ with embedding dimension $m= (k+2)^{k+2} \log(n)$. 
\end{restatable}
\paragraph{Proof sketch}We construct the MAS function by taking the one-hot embedding $F:\Pvk \to \RR^n$ from \eqref{eq:onehot}, and then applying $m$ random projections defined by vectors in $\RR^n$ with non-negative entries. The non-negativity ensures monotonicity, and we prove that for the stated value of $m$, the probability of achieving a MAS function using this procedure is strictly positive. 

The theorem shows  that $\mm(V,k)\leq (k+2)^{k+2}\log(n)$.
We now give two lower bounds on the embedding dimension: we show that $m$ must depend at least linearly on $k$, and at least double logarithmically on $n$:

\begin{restatable}[Lower bounds on $\mm(V,k)$]{theorem}{lowerBoundK}
\label{thm:lower_bound_k}
Let $k\geq 2$ and $n$ be natural numbers, and let $V$ be a ground set with $|V|=n$.  Then the smallest possible dimension $\mm(V,k)$ of a MAS function satisfied $\mm(V,k) \geq \log_2(\log_3 n)$. Moreover, if $k\leq \frac{n-1}{2} $ then $\mm(V,k)\geq 2k $
\end{restatable}
\paragraph{Proof Sketch}To obtain the first lower bound, we consider the sequence of singleton-set embeddings $\Mcal:= (F(\set{1}),\ldots,F(\set{n})$. By the Erd\"os-Szekers theorem~\cite{erdos}, every scalar sequence with $n$ elements has a monotone subsequence of length $\sim n^{\frac{1}{2}}$. Applying this recursively to vectors of length $m$ gets us a monotone subsequence in all $m$ coordinates of the vector, of length $\sim n^{\frac{1}{2^m}} $. If this subsequence is of length $\geq 3 $ then we will have three distinct elements $v_1,v_2,v_3\in V$ such that
$F(\{v_1\})[i]\leq F(\{v_2\})[i] \leq F(\{v_3\})[i] \text{ or } F(\{v_3\})[i]\leq F(\{v_2\})[i] \leq F(\{v_1\})[i], \forall i\in [m]$
Monotonicity implies that $F(\{v_2\})[i]$ is dominated by $F(\{v_1,v_3\})[i]$, for all $i\in [m]$, thus contradicting separability. It follows that MAS existence can only happen when $n^{\frac{1}{2^m}}\leq 2$, which leads to the double logarithmic lower bound. 

For the lower bound in terms of $k$, the argument is similar to the proof of ~\cref{thm:dim_geq_n}, where we select a  "maximal singleton element" $v^*_i$ for every dimension $i \in [m]$. Since $m < |V|-1$,  we can find elements $\set{u_1, u_2}$ disjoint from the collection $\bigcup_{i=1}^m\set{v^*_i}$, and WLOG $F(\set{u_1})[i] \geq F(\set{u_2})[i]$ for at least half of the indices in $[m]$. We then construct $S = \set{u_2}$ and $T$ to be union of $u_1$ and all the $v^*_i$ for all dimensions $i$, where $F(\set{u_1})[i] \leq F(\set{u_2})[i]$. Then monotonicity implies $F(S) \leq F(T)$, but $S \not \subseteq T$. This implies $k \leq |T|-1 \leq m/2$, as otherwise separability of $F$ in $\Pvk$ is violated. 

 \vspace{-5pt}
\subsection{Non-existence of MAS functions: for infinite ground set} 
For most practical applications, we deal with infinite (often uncountable) ground sets, for example $V=\R^d$. Thus, we would like to analyze whether MAS functions exists when $|V| = \infty$. The answer to this question is negative:
\begin{corollary}
    \label{impossibility}
    Given a ground state $V$ with $|V| = \infty$, and $k\geq 2$, there does not exist a MAS function  $F: \Pvk \to \R^m$ for any finite $m\in \NN$. 
\end{corollary}
This corollary is a simple consequence of the result in   Theorem~\ref{thm:lower_bound_k} that the embedding dimension is lower bounded by $\log_2 \log_3 |V|$ when $V$ is finite. We note that this result makes the minimal assumption $k\geq 2$. In Appendix \ref{app:proofs} we show this assumption is necessary, and that in the degenerate case $k=1$ there is a MAS function 
$F:\mathcal{P}_{\leq 1}(V) \rightarrow \mathbb{R}^{2}$ for $V=[-1,1] $.

\paragraph{Summary} In this section, we provided lower and upper bounds for the smallest value $\mm(V,k)$ for which an MAS function $F:\Pvk \to \RR^m $ exists. Our results are summarized in Table \ref{tab:m}. We note that some of these results require weak assumptions, which are stated in the theorems but not in the table.
\begin{table}[!!h]
    \centering
\resizebox{0.98\textwidth}{!}{
    \begin{tabular}{>{\centering\arraybackslash}m{6.2cm}|c||c|c} \hline  
        \multicolumn{2}{c||}{Ground set size $: \qquad \qquad$ $|V|=n < \infty$} & \multicolumn{2}{c}{$|V|=n = \infty$} \\ \hline
         Input multiset size $: \quad $  $k < \infty $ & $k=\infty$ & $k=1$ & $k\geq 2$ \\ \hline  \hline 
        $\mm(V,k)\leq \min\{n,(k+2)^{k+2}\log(n) \}$,
        $\mm(V,k)\ge \max\{2k,\log_2 \log_3 n\}$
       & $\mm(V,k)=n$ & $\mm(V,k)=2$ & Not possible \\  \hline
    \end{tabular}}
    \vspace{2mm}
    \caption{\it This table summarizes the lower and upper bounds for the smallest value $\mm(V,k)$ for which an MAS function $F:\Pvk \to \RR^m $ exists, where $k$ is maximal set cardinality and $n$ denotes the cardinality of the ground set $V$.  }
    \vspace{-6mm}
    \label{tab:m}
\end{table}
\vspace{-4pt}
\section{Relaxations of MAS functions}
\vspace{-4pt}
\label{sec:weak-mas}
Our definitions so far did not address differential parameterized set functions $F(S,w)$, which are the natural object of interest for set learning applications. One natural way to address this could be to require that there exists $w$ such that $F(\bullet,w)$ is a MAS function; however, since our theoretical results show that MAS functions do not exist when $|V| = \infty$, this requirement is too restrictive. Instead, we retain monotonicity as a hard constraint for every parameter $w \in \Wcal$, while treating separability as an existential condition over the parameter space. 
\vspace{-4pt}
\subsection{The notion of weakly \mas functions}
\vspace{-4pt}
Formally, given a ground set $V$, and a probability space $(\Wcal, \Bcal, \mu)$, we consider parametric set functions $F:\Pvinf \times \Wcal \to \mathbb{R}^m$, where $F(S,\bullet)$ is measurable for all $S$. We then define weakly \mas functions as: 
\begin{definition}[weakly  \mas function]
\label{def:weak-mas}
 The set function $F:\Pvinf\times \Wcal\to \mathbb{R} $ is a weakly  \mas function if the following two conditions are satisfied:\\
\textbf{(1) Pointwise monotonicity}  For any $w\in W$ and $S \subseteq T$, we have that  $F(S,w) \leq F(T,w)$.\\
\textbf{(2)  Weak separability}
If $S \not \subseteq T$, then there exists $w\in \Wcal$ such that  $F(S;w)>F(T;w)$. 
\end{definition}

\paragraph{Which condition, monotonicity or separability, should be relaxed?} In our definition, we relax the separability requirement so that different multiset pairs can be separated by different parameters, while we maintain strict requirements of pointwise monotonicity. This is because, as discussed in Example \ref{ex:simple}, constructing monotone set functions is straightforward, even in one dimension. In contrast, separability is a much stronger and more difficult condition to satisfy. This is a key reason why \mas functions do not exist when $|V| = \infty$. Indeed, as shown in Appendix~\ref{app:proofs}, there can be no continuous separable set function even without the monotonicity assumption when the ground set is $V = \RR^D$.
Therefore, separability is the natural condition for relaxation. Moreover, monotonicity remains a useful and often necessary constraint in many applications.




\paragraph{Are scalar deep sets weakly  MAS?}
Equipped with the new notion of weakly \mas functions, we ask the question of how to construct such a class of functions. As a first step to achieve this goal, we look into a simple instance of DeepSets \cite{zaheer2018deepsets}. DeepSets defines a set function by applying an 'inner' MLP $\mone$ to each set element, summing over the result, and then applying an 'outer' MLP $\mtwo$. We consider the case where $\mone$ is a shallow MLP $\mone(x)=\sigma(Ax+b)$ which gives us models of the form 
\begin{align}
F(S; (A,b)) =\textstyle \mtwo\big(\sum_{x\in S} \sigma(A x +b)\big)    \label{eq:parametric_setfn}
\end{align}
To ensure monotonicity for every parameter choice $A,b$, we will require that $\mtwo$ is a monotone vector-to-vector function, and the activation $\sigma$ is non-negative. While these choices automatically ensure monotonicity, they may not lead to weak MAS functions as for many activations there will be no separation:
\begin{proposition}
\label{prop:shallow_relu_not_sep}
   If the deep set model in \eqref{eq:parametric_setfn} is implemented with a vector-to-vector monotonously increasing function $\mtwo$ and a non-negative activation $\sigma$, then $F(\bullet; (A,b))$ is monotone for every $(A,b)$. If in addition $\sigma$ is monotone (increasing or decreasing), then there exists $S \not \subseteq T$ with $F(S; (A,b))\leq F(T; (A,b)) $ for all $A,b$.   
\end{proposition}
\textbf{Proof idea}
The monotonicity is rather straightforward. To prove lack of separation when $\sigma$ is monotone, choose 
$x\neq y, \ z=\frac{1}{2}(x+y), \ S=\{z\}, \ T=\{x,y\}, $
and then $Az+b$ is the average of $Ax+b$ and $Ay+b$, which can be used to show that $F\left(T;(A,b)\right)\geq F\left(S;(A,b)\right)  $ for all $A,b$.

\paragraph{Are Set Transformers weakly MAS?}Now we ask the quedtion if Set Transformers~\cite{lee2019settransformerframeworkattentionbased} are weakly MAS. We notice in the following that Set Transformers are not even monotone in the following result. For this, we consider Set Transformer with sum-pooling, namely $F(S; W_Q, W_K, W_V) = \text{SumPool}(\text{Attn}(S))$. Explicity, this is defined for a give multiset $X = \set{\bx_1, \cdots, \bx_n}$ by first defining query, key, values per-point: $\bq_i = W_Q\bx_i, \bk_i = W_K\bx_i, \bv_i = W_V\bx_i$, and then taking a weighted average, defined by weights: $\alpha_{i, j} = \frac{e^{\bq_i\cdot \bk_j}}{\sum_s e^{\bq_i \cdot \bk_s}}$ to obtain a permutation-invariant function $F(S; W_Q, W_K, W_V):= \sum_{i,j}\alpha_{i,j}\bv_j$

\begin{restatable}[]{proposition}{STnonmon}
    Let $V = \R^d$ and let $F: \Pvinf \to \R^D$ be a Set Transformer defined by $F(S; W_Q, W_K, W_V) := \text{SumPool}(\text{Attn}(S))$. Then $F$ is not a point-wise monotone function.
\end{restatable}
\vspace{-5pt}
\paragraph{Proof Idea} We consider full rank matrices $W_Q = W_K$ and any non-zero $W_V$. We choose $\bx_1 \in \R^d$ such that $\bv_1 := W_V \bx_1 \neq 0$, and consider the sets $S = \set{\bx_1}, T = \set{\bx_1, \mathbf{0}_d}$. Then, $F(S; W_Q, W_K, W_V) = \bv_1$ and $F(T; W_Q, W_K, W_V) = (\alpha_{1,1} + \alpha_{1,2})\bv_1$. Note that, $\alpha_{1,1} + \alpha_{1,2} > 1$ and thus, for any negative $\bv_1[j]$, we have $F(T; W_Q, W_K, W_V)[j] < F(S; W_Q, W_K, W_V)[j]$ and thus, monotonicity is violated.

Thus, unlike Deep Sets, it's not obvious how to make Set Transformers monotone, let alone weakly separable. Thus, in the following sections, we shall focus on obtaining weakly-MAS functions from DeepSet-like models only. 
\vspace{-4pt}
\subsection{The \hhat activation class}
\vspace{-4pt}
Note that most commonly used non-negative activation functions, such as $\relu $ or sigmoid, are monotonically non-decreasing and therefore, from Proposition~\ref{prop:shallow_relu_not_sep}, using them as $\sigma$ above will fail weak separability. 
Here, we consider a novel class of activation functions we call  'Hat Activations', which will all make $F$  in Eq.~\eqref{eq:parametric_setfn} weakly MAS. 

\begin{definition}[The \hhat\  activation function]
\label{def:hhat}
We call a function $\act: \R \to \R$  a hat activation if it is (a) non-negative  (b) compactly supported (c) not identically zero and (d) continuous. 
\end{definition}\vspace{-6pt}
Examples of hat functions are the third and fourth functions in Figure \ref{fig:activations}. 
\begin{figure}
    \centering
   \includegraphics[width=\linewidth]{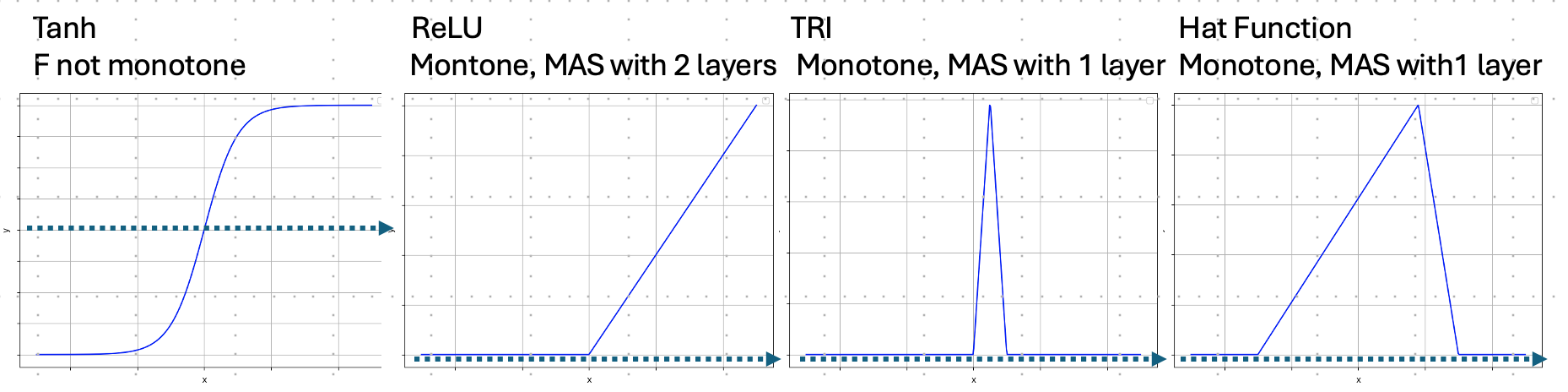}
    \caption{\it Using  a multiset model as in \eqref{eq:parametric_setfn} with activations which are not always non-negative, like $\sigma=\mathrm{Tanh}$ will not be monotone. ReLU will be monotone, but to be weakly MAS, two layers are required. TRI and more general hat functions are weakly MAS even with a single layer.}
    \label{fig:activations}
\end{figure}
When $\sigma$ is a hat function, \cref{eq:parametric_setfn} is weakly MAS, even when the output of $F$ is taken to be scalar. For this result, we will need to require that $\mtwo$ is strictly monotone. This can be handled by simply setting $\mtwo$ to be the identity. 
\begin{proposition}
    \label{prop:weakmas_hat}
    If $\mtwo$ is strictly monotone increasing, and $\sigma$ is a hat activation function, then $\mtwo(\sum_{x\in S}\sigma(a\tran x+b)) $ is weakly \mas function. 
\end{proposition}
\paragraph{Proof idea}
If $S$ is not a subset of $T$, there is an element $s$ whose multiplicity in $S$ is larger than its multiplicity in $T$. We can then choose $a,b$ so that $s$ is in the support of  $\sigma(ax+b)$ but the elements of $T$ are not.  
\vspace{-4pt}
\subsection{Slightly larger ReLU networks}
So far, we have considered simple models as in \eqref{eq:parametric_setfn} which apply a single linear layer before applying activation and summation. In this case, we saw that $\relu$ and other activations do not create weakly \mas functions, and defined the notion of hat activations, which are weakly \mas functions. But what happens if we use $\relu$ activations but allow deeper networks? 

We can use our previous analysis to show that very small $\relu$ networks with two layers can already be \mas functions. To see this, we consider a candidate from the hat function class, namely the $\tri$ function from Figure \ref{fig:activations},  and observe that: $\tri(x)=  \relu\left(\relu(2x)+\relu(2x-2)-\relu(4x-2) \right)$
We use this to deduce the following result:

\begin{restatable}[]{proposition}{proprelumas}\label{prop:relu_mas}
Given $V \subseteq \R^d$, we consider affine transformations $\mathcal{A}_2:\RR^d \to \RR^3, \mathcal{A}_1:\RR^3 \to \RR$. Thus, here $\mathcal{A}_2(t) := A_2 t + b_2$ and $\mathcal{A}_1(z) = a_1\tran z + b_1$, where $A_2 \in \R^{3 \times d}, b_2 \in \R^3,a_1 \in \R^3, b_1 \in \R$ are the respective parameters.  Then, the set functions of the form 
\begin{equation}
    \label{eq:relu_masnet}
    F: \Pvinf \to \R, \textstyle F(S;\mathcal{A}_1,\mathcal{A}_2)=\mtwo\left(\sum_{x\in S} \relu \circ \mathcal{A}_1 \circ \relu \circ \mathcal{A}_2(x) \right)
\end{equation}
 are weakly \mas functions.
\end{restatable}

Our results so far are summarized in Figure \ref{fig:activations}. If $\sigma$ is not a non-negative function, then $F$ will not be monotone. We can attain weakly MAS functions with a two-layer ReLU network, or with a one-layer network with hat activations. 
\subsection{Asymmetric distance induced Holder separability}
\paragraph{Why weak separability is not sufficient} 
While weak separability is a key relaxation for weakly \mas functions~(Definition~\ref{def:weak-mas}), it has two key limitations:
\begin{enumerate*}[label=\textbf{(\arabic*)}]
\item Separation should ideally hold over a non-negligible subset of $\Wcal$, not just a single point $w$;  \item it treats separability as a boolean condition, ignoring approximate containment. We would like to guarantee that if $S$ is almost a subset of $T$, then $F(S,w)$ will be close to dominated by $F(T,w)$.
This motivates us to create an \emph{asymmetric (pseudo) metric} to quantify the size of set difference and create stronger notions of separability based on that.

\end{enumerate*}

\paragraph{Asymmetric set distance}
For a ground set $V \subseteq \R^d$ and two sets $S, T \in \Pvinf$ such that $|S| \leq |T|$.
We define the asymmetric distance from $S$ to $T$ as the Earth Mover Distance (EMD) between $S$ and the subset of $T$ that is \emph{closest} to $S$ with the same cardinality as $S$.
\begin{equation}
    \label{eq:asymmetric_dist}
    \das(S, T) =  \min_{T'\subseteq T: |S|=|T'|}\emd(S,T')
\end{equation}
$\das(\bullet,\bullet)$  captures graded set containment:  if $S \not \subseteq T$, but is "very close" to a subset of $T$ in EMD metric, then $\das(S,T)$ is small. Moreover, $\das(S, T) = 0 \iff S \subseteq T$. Based on $\das$, we propose a newer separability notion for parametric functions.

\paragraph{Lower H\"older separability of set functions}
\citet{davidson2025holderstabilitymultisetgraph} introduced H\"older continuity through expectation over parameters. Unlike their symmetric EMD, we define H\"older separability using the asymmetric distance $\das$\eqref{eq:asymmetric_dist}.
\begin{definition}[Lower H\"older separability]
    \label{def:lower-holder}
    Let $V \subseteq \R^d$ be the ground set and $(\Wcal,\Bcal,\mu)$ be a probability space, and a constant $\lambda>0$.  A parametric set function $F(\bullet,w) : \Pvk \to \R^m$ with $w\sim \mu(\cdot)$ is $\lambda$ lower H\"older separable if there exists $c > 0$ such that:
    \begin{equation}
        \label{eq:lower_holder}
         \Exp_{w \sim \mu(\cdot)}\norm{\sbrac{F(S; w) - F(T; w)}_+}_1 \geq c \cdot \das(S, T)^\lambda, \\ \ \text{for all } S, T \in \Pvk
    \end{equation}
\end{definition}

\subsection{H\"older separable Set functions}
\label{sec:holder_separable_discussion}
\paragraph{H\"older separability using \hhat activation}Proposition \ref{prop:weakmas_hat} shows how to construct  weakly \mas\ functions using \hhat activations.
We now show that these constructions are also lower H\"older, under additional weak assumptions on the hat functions and $M_2$: 

 \begin{restatable}[$F$ is lower H\"older]{theorem}{flowerholder}
   \label{thm:F_lower_holder}
     Let $V \subset \R^d$ be a compact set, and $\sigma$ a \hhat\ activation function which is piecewise continuously differentiable supported in some interval $[\leftend, \rightend]$, and satisfying the condition: $\lim_{t \to \leftend^+}\diff{\sigma}{t} > 0 $. Let  $\mtwo: \R \to \R$ be a lower Lipschitz function. Consider the function
     $F(S; \tuple{a, b, c})= \mtwo\big(\sum_{x \in S}\sigma\big(\frac{a\tran x + b}c\big)\big),$ 
     where the multisets $S$ come from $\Pvk$, and $a \sim \Unif(\S^{d-1}), b \sim \Unif([-1,1]), c \sim \Unif((0,2])$. Then, $F(\bullet,\tuple{a, b, c})$ is Monotone H\"older separable with exponent $\lambda=2$.
\end{restatable}
\vspace{-2pt}

\paragraph{Probability of successful separation} Given $S, T \subseteq V$, we can consider $[F(S; \tuple{a, b, c}) - F(T; \tuple{a, b, c})]_+$ to be a real-valued non-negative random variable, where the randomness is over the parameter space $\tuple{a, b, c}\in \Wcal =\RR^d \times \RR\times \RR$ equipped with a probability measure. The above result on lower H\"older stability 
gives us a lower bound on the probability of a parameter tuple $\tuple{a, b, c}$ separating two non-subsets $S, T$ proportional to the set distance $\das(S, T)$. Moreover, the functions we have considered so far are scalar-valued set functions. By considering independent copies of the parameters across multiple dimensions, we can increase the probability of separation, since $F(S), F(T)$ are not separated iff $F(S)[i] \leq F(T)[i]$ across each embedding dimension $i$.  We formalize the above observations as follows:

\begin{restatable}[Probability bounds on separation]{theorem}{probabilityseparation}
\label{thm:probabilistic_implications_of_lower_holder}
 Let $V \subseteq \R^d$ and $\sigma$ be as in Theorem \ref{thm:F_lower_holder}, and let $A\in \RR^{m\times d},b\in \RR^m, c\in \RR^m$ whose $m$ rows (respectively entries) are drawn independently from the distribution on $a_j,b_j,c_j$ described in Theorem \ref{thm:F_lower_holder}, and consider the function 
 \begin{equation} \label{eq:hadamard}
 \textstyle F(S;A,b,c)=\sum_{x\in S} \sigma\left(  c^{-1} \odot (Ax+b)\right) .\end{equation}
Then there exists $C>0$, so that for all $S \not \subseteq T$, 
$\Prob\left( F(S)\leq F(T) \right) \leq \round{1 - C\das^2(S,T)}^m.$
 \end{restatable}
In Equation \eqref{eq:hadamard} $c^{-1}$ stands for the elementwise inverse and $\odot$ stands for elementwise multiplication. 

Note that the probability of failed separation goes to zero exponentially as $m$ increases, and also becomes smaller as $\das(S,T)$ increases. This supports the idea that a larger measure of parameters separates sets with larger asymmetric distance and that separability becomes easier with larger embedding dimensions (which we saw previously in the existential results)

\paragraph{H\"older separable set functions using ReLU} As seen in Proposition~\ref{prop:relu_mas}, one can construct weakly \mas functions with a two layer neural network with $\relu$ activation. Our proof used that two layer $\relu$ networks can basically "simulate" a \hhat function. Under the light of the above results on H\"older separability, one can argue that two-layer $\relu$ networks would also have such guarantees. Also, we show in Appendix~\ref{app:proofs} that under certain structure on the ground set $V$ (for example, $V$ being a hypersphere), even with \emph{one layer} $\relu$ networks, functions of the form in ~\cref{eq:parametric_setfn}  will also have Lower H\"older separability guarantees. 

\paragraph{Upper Lipschitz bounds} In Appendix.~\ref{app:proofs} we complement our results on \emph{lower} Holder stability by showing that the construction in Theorem \ref{thm:probabilistic_implications_of_lower_holder} is also \emph{upper} Lipschitz in an appropriate sense. 
\vspace{-4pt}
\subsection{Monotone Universality and the role of $M_2$}
\vspace{-4pt}
Besides set containment-based applications, there are other interesting scenarios where one would like to construct monotone multiset-to-vector functions. The following theorem shows that on finite ground sets,  the combination of a \mas function and a universal monotone vector-to-vector (such as \cite{sill1997monotonic})  can approximate \emph{all} monotone multiset-to-vector functions. 

\begin{restatable}[Universality]{theorem}{universal}\label{thm:universal_approx}
Let $V$ be a finite ground state, and let $F:\Pvk \to \RR^m$ be a \mas\ function. Then for every 
multiset-to-vector monotone function $f:\Pvk \to \RR^s $, there exists a vector-to-vector monotone function $M:\RR^m \to \RR^s$ such that $F(S)=M(f(S)) $.
\end{restatable}
\vspace{-10pt}
\section{\our: Neural Modeling of \mas functions}
\vspace{-5pt}
\label{sec:model}
Now, our goal is to leverage our theoretical analysis to design neural network-based multiset-to-vector models that preserve monotonicity and weak separability. Motivated by the formulation in DeepSets and our analysis so far, we wish to design neural set functions of the form:  
\begin{equation}
    \label{eq:masnet}
    \our(S) = \textstyle\mtwoneural\round{\sum_{x \in S}\sigma\round{\moneneural(x)}}
\end{equation}
In our analysis, we required $\mtwoneural$ to be a monotone vector-to-vector function. In most of our experiments, we enforce this simply by choosing $M_2(x)=x$, but this can also be enforced by using monotone activations and non-negative parameters. 
\vspace{-0.5pt} 

\paragraph{\hatmas} To enforce weakly \mas in~\cref{eq:masnet},our results from Prop.~\ref{prop:weakmas_hat} and Theorem~\ref{thm:F_lower_holder} suggests using a \hhat activation as $\sigma$. In our experiment, we do it by not choosing $\sigma$ to be a specific hat activation, but rather the parametric form:    
\begin{equation}\label{eq:simpleHat}
\textstyle \sigma_{\alpha,\beta,\gamma}(x)=\relu \round{ \frac{x-\alpha}{\gamma \cdot \beta}}
+\relu \round{ \frac{x-(\alpha+\beta)}{(1-\gamma) \cdot \beta}}
-\relu \round{ \frac{x-(\alpha+\gamma \cdot \beta)}{\gamma \cdot (1-\gamma) \cdot \beta}}
\end{equation}
For all $\alpha$, $\beta>0$, and $\gamma\in (0,1)$, $\sigma_{\alpha,\beta,\gamma}$ is a hat function (Definition \ref{def:hhat}) with support $[\alpha,\alpha+\beta]$ and peak at $\alpha+\gamma\cdot\beta$. Examples include the third and fourth functions in Figure \ref{fig:activations}. Applying this to an $m$-dimensional output with independent $\alpha, \beta, \gamma$ per dimension yields $3m$ parameters total.
\vspace{-0.5pt}

\paragraph{\our with $\relu$}
Motivated by our theoretical results from ~\cref{prop:relu_mas} of 2-Layer $\relu$ networks being weakly \mas and from our discussion of $\relu$ networks being H\"older separable (under appropriate assumptions) in Section~\ref{sec:holder_separable_discussion}, we propose use $\relu$ in \our, which case $M_{\theta_1}$ will have to be a network with at least 2 layers to ensure weakly \mas as in \cref{prop:relu_mas}. This gives one more way to design \our and we refer to this option as \relumas.
\vspace{-0.5pt}

\paragraph{Other variants of\ \our} Definition \ref{def:hhat} establishes a general hat function class extending beyond the piecewise linear functions as in \eqref{eq:simpleHat}. We present \integralmas in Appendix \ref{app:model}, which achieves universal approximation of the \hhat class by modeling its derivative via neural networks and approximating the integral. Additionally, using the specific hat activation $\tri$ (Figure \ref{fig:activations}) as $\sigma$ yields an alternative formulation called \trimas.
\vspace{-0.5pt}

\paragraph{Which \our to choose and when?} We have given multiple recipes of \our in the above discussion, and the question arises: which one to choose? We show via our experiments: For parameter-constrained scenarios requiring single-layer $M_{\theta_1}$ networks, \hatmas provides optimal separability (both theoretically and empirically). For deeper $M_{\theta_1}$ ($\geq 2$ layers), both \relumas and \hhat-based variants perform similarly, with \relumas often providing more stable training and better performance overall.

\vspace{-5.8pt}
\section{Experiments}
\vspace{-3.8pt}
 \label{sec:expts}
We evaluate the \our\ variants (Section~\ref{sec:model}) on synthetic, text and point-cloud datasets to characterize monotonicity and separability. Specifically, we focus on (exact and approximate) set containment. Given sets $S$ and $T$, with binary label $y(S, T) \in \{0, 1\}$ indicating (exact or approximate) set containment, we evaluate how accurately \our\ predicts $y(S, T)$.

\vspace{-4pt}
\subsection{Set Containment}
\vspace{-3pt}
We perform experiments on a synthetically generated dataset, four text datasets, and one image dataset. In each case, we split the dataset into 5:2:2 train, test, and dev folds. We minimize the following fixed-margin hinge loss that enforces vector dominance, to train the parameters of  \our.
\begin{align}
     \sum_{S,T} (1-y(S,T)) \min_{i\in[m]}\sbrac{\model(T)[i] - \model(S)[i] +\delta}_+ 
    + y(S,T)\max_{i\in[m]}\sbrac{\model(T)[i] - \model(S)[i] +\delta}_+ 
\vspace{-2mm}
\end{align}
 %

\begin{wraptable}[9]{r}[-10pt]{0.42\linewidth}
    \vspace{-3mm}
    \centering
    \captionsetup{font=small, skip=3pt}

    \begin{adjustbox}{max width=\linewidth}
        \begin{tabular}{@{}cc@{\hspace{9pt}}cccc@{}}
            \toprule
            \multicolumn{2}{c}{\textbf{Set sizes}}
            & \multicolumn{4}{c}{\textbf{Model}} \\
            \cmidrule(lr){1-2}
            \cmidrule(l){3-6}
            \(\lvert S\rvert\)
            & \(\lvert T\rvert\)
            & DS
            & ST
            & M-ReLU
            & M-\(\hhat\) \\
            \midrule
             1 &   2
               & \underline{0.98}
               & \underline{0.98}
               & \textbf{0.99}
               & \textbf{0.99} \\

             1 &  10
               & 0.59
               & 0.59
               & \textbf{0.99}
               & \textbf{0.99} \\

            10 &  30
               & \underline{0.89}
               & 0.80
               & \textbf{1.00}
               & \textbf{1.00} \\

            10 & 100
               & 0.54
               & 0.57
               & \underline{0.98}
               & \textbf{0.99} \\
            \bottomrule
        \end{tabular}
    \end{adjustbox}

    \caption{Accuracy on the synthetic dataset.}
    \label{tab:synth}
    \vspace{-2mm}
\end{wraptable}

\paragraph{Analysis on synthetic datasets} To generate our synthetic dataset, we first sample a target set $T \subset \R^d$, where each element  $x\in T$ is drawn from $\Normal(0, \bm{I})$, i.i.d. Given $T$, we compute $S$ with a fixed size $|S|=s$ as follows. We first obtain $S$ with positive labels $y(S,T)=1$ by drawing a subset from $T$, uniformly at random, without replacement. To generate $S$ with $y(S,T)=0$, we sample $s$ from $\Normal(0,\bm{I})$ independently. We chose $d=4$ and the set embeddings have $m=256$ dimensions, the class-ratio of subsets to non-subsets was taken to be 1:1. 

We compare two variants of \our--- \our-ReLU (M-ReLU) and \our-Hat (M-Hat)--- against Deep Sets (DS) \cite{zaheer2018deepsets} and Set Transformer (ST) \cite{lee2019settransformerframeworkattentionbased}. Table~\ref{tab:synth} summarizes the results
\begin{wrapfigure}[11]{r}{0.37\textwidth} 
\vspace{-1mm}
    \centering
     \includegraphics[width=0.35\textwidth]{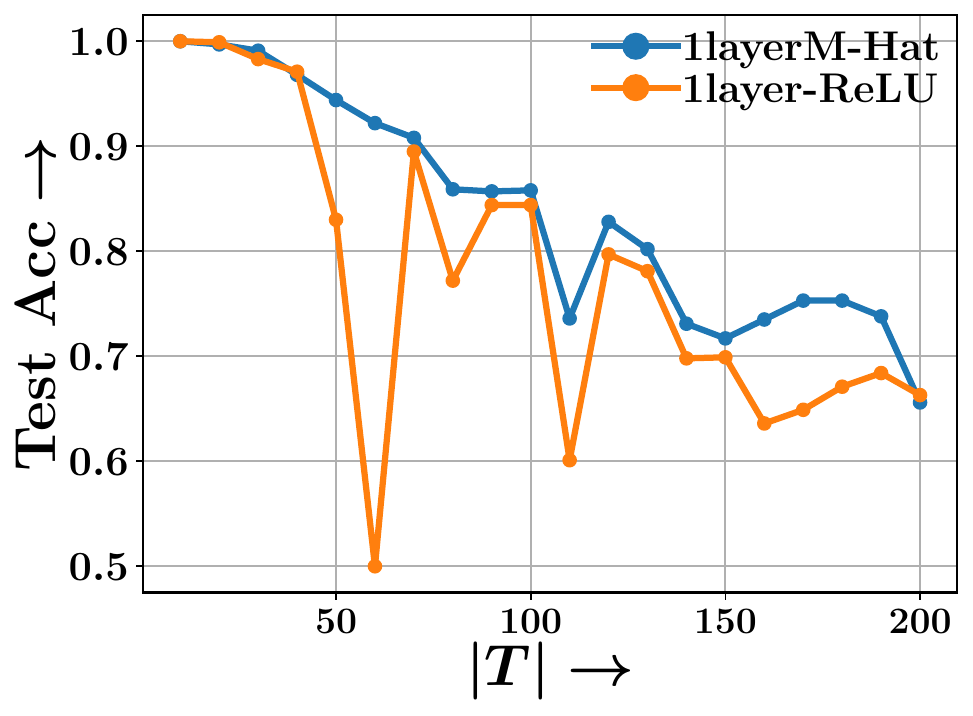}
    \caption{Acc vs $|T|$ for 1-layer MLP}
    \label{fig:shallow}
    \vspace{-15pt}  
\end{wrapfigure}
for varying values of $|S|$ and $|T|$. We observe that:
(1) Both variants of our method outperform the baselines;
(2) the baselines degrade significantly as $|T|$ increases (with fixed $|S|$), as predicting separability becomes harder as the gap between $|T|$ and $|S|$ increases;
(3) The performance of M-ReLU and M-Hat is comparable. We now compare 1-layer pointwise models ending 
 in $\relu$ and \hhat activation respectively, followed by aggregation: we call them 1layer-$\relu$ and 1layerM-\hhat. We see in Fig.~\ref{fig:shallow} that: (1) 1layerM-\hhat performs significantly than 1layer-$\relu$, validating the results from Propn.~\ref{prop:shallow_relu_not_sep}. (2) acc. declines for both as $|T|$ increases, consistent with the intuition that larger $|T\setminus S|$ gaps make the task harder.

\paragraph{Set containment on Text datasets} We evaluate on MSWEB, MSNBC, and Amazon Registry datasets, which exhibit natural set containment from user behavior. Each datapoint is a wordbag; we compute BERT embeddings for all unique words to form the ground set. $(S, T)$ pairs are sampled from these  (Appendix~\ref{app:addl-details}), and labeled via set containment. For \inexact containment,  small 
\begin{table}[t]
    \centering
    \small
    \captionsetup{font=small, skip=3pt}
    \setlength{\tabcolsep}{4pt}
    \renewcommand{\arraystretch}{1.05}

    \begin{tabular}{@{}lcccc@{}}
        \toprule
        \textbf{Model}
        & \textbf{Bedding}
        & \textbf{Feeding}
        & \textbf{MSWEB}
        & \textbf{MSNBC} \\
        \midrule
        DeepSets
        & 0.51 & 0.50 & 0.88 & 0.67 \\

        Set Transformer
        & 0.77 & 0.79 & 0.90 & 0.94 \\

        FlexSubNet
        & 0.88 & 0.85 & 0.92 & 0.91 \\

        Neural SFE
        & 0.52 & 0.53 & 0.88 & 0.66 \\

        \midrule
        \relumas
        & \textbf{0.98}
        & \textbf{0.98}
        & \textbf{0.99}
        & \textbf{0.97} \\

        \hatmas
        & \underline{0.94}
        & \underline{0.93}
        & \underline{0.97}
        & \underline{0.95} \\
        \bottomrule
    \end{tabular}

    \caption{Accuracy on text datasets with a 1:1 class ratio.}
    \label{tab:noisy_accuracy}
\end{table}

Gaussian noise is added to $S$ without changing $\labelst$.
We now compare the variants of \our: \relumas and \hatmas against several baselines, DeepSets~\cite{zaheer2018deepsets}, Set Transformers~\cite{lee2019settransformerframeworkattentionbased}, FlexSubNet~\cite{de2022neuralestimationsubmodularfunctions}, Neural SFE~\cite{sfekaralias2022neural}, for \inexact set containment on the datasets: MSWEB, MSNBC, Amazon-Feeding, Amazon-Bedding. Gaussian noise with $0.01$ std was used to generate noisy pairs. The class ratio of subsets to non-subsets was taken to be 1:1.  Here, set embeddings have dim $m=50$, and set elements have dim $d=768$. The results in table~\ref{tab:noisy_accuracy} show: (1) \relumas and \hatmas have similar acc., with \relumas being marginally better (2) Both are generally better than baselines.

\paragraph{Set containment on point clouds} We use ModelNet40~\cite{modelnetwu20153dshapenetsdeeprepresentation}, a dataset of 3D CAD models across 

\begin{wraptable}{r}{0.40\textwidth}
    \vspace{-4pt}
    \centering
    \scriptsize
    \captionsetup{font=small, skip=2pt}

    \begin{adjustbox}{max width=\linewidth}
        \begin{tabular}{@{}lccc@{}}
            \toprule
            \textbf{Model}
            & \multicolumn{3}{c}{\(\lvert S\rvert\)} \\
            \cmidrule(l){2-4}
            & \(128\) & \(256\) & \(512\) \\
            \midrule

            DeepSets
            & 0.52 & 0.51 & 0.50 \\

            Set Transformer
            & 0.62 & 0.51 & 0.50 \\

            FlexSubNet
            & 0.84 & 0.75 & 0.60 \\

            Neural SFE
            & 0.52 & 0.51 & 0.50 \\

            \midrule
            \relumas
            & \textbf{0.98}
            & \textbf{0.94}
            & \textbf{0.72} \\

            \hatmas
            & \underline{0.87}
            & \underline{0.81}
            & \underline{0.65} \\
            \bottomrule
        \end{tabular}
    \end{adjustbox}

    \caption{Point-cloud accuracy for different set sizes \(\lvert S\rvert\), using a 1:1 class ratio.}
    \label{tab:pointcloud}
    \vspace{-4pt}
\end{wraptable}

40 categories. For each object in category $C_1$, we sample 1024 points to form $T$. For true subsets, we sub-sample $S \subseteq T$ and for non-subsets, $S$ is sampled from a different category $C_2$.  Small white noise(std=0.005) is added to $S$ for the task of \inexact containment. We use a pointcloud encoder(such as pointnet) followed by a set-to-vector embedding model to produce $m=50$ dimensional embeddings, and each point in the dataset is a 3D coordinate with $d=3$ dimensions. The class ratio of subsets to non-subsets was taken to be 1:1.  Results are in table~\ref{tab:pointcloud}. We observe that: (1) \relumas and \hatmas are generally better than baselines, with \relumas being slightly better. (2) The accuracy of \our decreases as $|S|$ decreases, with the separability becoming hard for monotone models.
 %

\vspace{-3pt}
\subsection{Monotone Set function approximation}
\vspace{-5pt}
\begin{wraptable}[5]{r}{0.40\textwidth}
    \vspace{-5mm}
    \centering
    \scriptsize
    \captionsetup{font=small, skip=2pt}
    \setlength{\tabcolsep}{6pt}
    \renewcommand{\arraystretch}{1.05}

    \begin{tabular}{@{}lc@{}}
        \toprule
        \textbf{Model} & \textbf{MAE} \\
        \midrule
        DeepSets
        & \(0.01024 \pm 0.00030\) \\

        Set Transformer
        & \(0.01026 \pm 0.00033\) \\

        \relumas
        & \(\underline{0.01023 \pm 0.00107}\) \\

        \hatmas
        & \(\mathbf{0.00959 \pm 0.00016}\) \\
        \bottomrule
    \end{tabular}

    \caption{Monotone function approximation.}
    \label{tab:mae_results}
    \vspace{-5pt}
\end{wraptable}

In this experiment, we examine if a scalar variant of \our\ can approximate a monotone function $F^*$, by getting trained from set, value pairs in the form of $\set{(S,F^*(S))}$. Once trained, we measure $|\mas(S)-F^*(S)|$ for each $S$ and average over them to compute MAE. 
We compare \our against baselines: DeepSets, SetTransformer. We trained using MSE loss, and report test MAE results in Table \ref{tab:mae_results} where we see: (1) \hatmas is best performing, followed by \relumas (2) Both are better than baselines.
\vspace{-5pt}
\section{Conclusion and Future Work}
\vspace{-5pt}
In this work, we study the design of set functions that are useful in set containment through Monotone and Separating (MAS) properties. We derive bounds on the embedding dimensions required for MAS functions and show their nonexistence in infinite domains. To address this, we introduce a relaxed model, \our, which satisfies a weak MAS property and is provably stable. Experiments demonstrate that \our\ outperforms standard models on set containment tasks by leveraging monotonicity as an inductive bias. It would be interesting to check if our method can be used for applications in graphs, \eg, subgraph matching. Another potential direction is to consider other possible relaxations of separability.
\vspace{-5pt}
\section{Limitations}
\vspace{-5pt}
\label{app:impact}
Our work provides a theoretical analysis of monotonicity and separability for set functions in the context of set containment, and introduces neural models that outperform standard set-based architectures such as DeepSets and Set Transformers on set containment tasks. Our work has the following limitations, addressing which are interesting directions for future work:

\begin{itemize}
    \item After proving the impossibility of constructing \mas functions over infinite ground sets, we proposed parametric set functions defined with respect to a probability measure over the parameter space. We provided probabilistic guarantees: for non-subsets $S, T$, if their asymmetric distance $\das(S, T)$ is large, then a randomly sampled function from the parameter space will separate them with high probability. However, these guarantees hold only in the randomized setting---analogous to guarantees for neural networks at random initialization---and do not extend to models after training via gradient-based optimization.

    \item Our universality result, which shows that \mas functions composed with monotone vector-to-vector mappings can approximate all monotone set-to-vector functions, was only established for finite ground sets. Extending this result to infinite ground sets remains an open problem.

    \item We currently do not generalize our analysis to the subgraph isomorphism problem, which is a natural and strictly harder extension of the set-containment task.
\end{itemize}
\vspace{-5pt}
\section*{Acknowledgment}
\vspace{-5pt}
Nadav and Yonatan were supported by ISF grant 272/23,
Soutrik and Abir were supported by Amazon and Google Research Grant, as well as Bhide Family Chair Endowment Fund.

\newpage

\bibliography{refs}
\bibliographystyle{unsrtnat}
\newpage
\appendix

\begin{center}
\textbf{\large Appendix}
\end{center}
\section{More discussions on related work}
\label{app:related_works}
This work lies at the intersection of representation learning for sets and multisets, monotone and submodular set functions, and theoretical limits of learning under constraints on some partial order. Below we review related literature across several communities, such as machine learning, optimization, theoretical computer science, fuzzy measures, and physics. We highlight how these threads connect to our problem setting.

Multiset to vector functions, and permutation invariant functions (including symmetric and anti-symmetric polynomials) have been used in a bunch of real life modelling tasks. Anti symmetric polynomials in the context of quantum many-body physics was used in~\cite{huang2021geometrybackflowtransformationansatz}. For point cloud datasets, such set-to-vector embeddings are widely used such as in~\cite{qi2017pointnet}. In the context of search and retrieval, set-containment was directly used in~\cite{roy2023locality}. On a somewhat related aspect, \cite{singh2020explaining} deals with monotone ranking functions on a vector of This is specifically a monotone function from multisets to reals if the vector P denotes the weight/frequency of each object in the set. Moreover, since this is a vector-to-real monotone function, composing a set-to-vector monotone function with a monotonic ranking function preserves monotonicity. These are relevant to the recent usage of set containment in the context of information retrieval. Another related usage is in vector search~\cite{engels2023dessert}, where relevance based retrieval is done after converting objects into high-dimensional embeddings. 

More specifically monotone (multi)set to vector functions have been of interest in the form of theory of capacities~\cite{choquet1953capacities} which is a probability measure based on monotonicity. In fuzzy set theory, where partial membership is allowed, any set-to-vector monotone function is equivalent to a vector-to-vector monotone function. Measures and integrals on such structures have been explored in the form of Choquet integrals~\cite{agahifuzzy, monotonefuzzy} etc. In the context of game theory and economics, monotone and submodular functions are used to model diminishing marginal utilities and in combinatorial auctions, such as in~\cite{lehmann2001combinatorial,dobzinski2005approximation,feige2009maximizing}. In learning theory, understanding PAC learning and hardness of vector-to-vector monotone functions have always received a lot of interest. And using the membership vector of multisets, this id bassically equivalent to learning multiset-to-vector monotone functions. Some works such as~\cite{blum1988training} uses set contaionment and monotone set functions as a tool in the proof. Specifically, they perform a reduction of neural network training to the set splitting problem to prove the hardness result.This is the task of partitioning a set $S = S_1 \cup S_2$ such that for a collection of subsets $\mathcal{C} = \{C_i: C_i \subseteq S, \forall i\}$ we must have $C_i \not\subseteq S_1$ and $C_i \not\subseteq S_2$, $\forall i$. This is basically preventing set containment for a subset family. More direct applications include~\cite{you2017deep} which discusses the challenges of learning vector to vector monotonic functions or~\cite{balcan2012learning} that discusses learning monotone valuation functions. Works like~\cite{bshouty1993exact, odonnell2003learning} deal with learning monotone boolean functions to any constant accuracy, under the uniform distribution in polynomial time. Since any boolean vector is basically one-hot encoding of a set on a fixed size vocabulary, this paper essentially describes learning monotone set functions. Works on monotone learning such as~\cite{li2025monotonic} describe learning processes in which expected performance consistently improves
as the amount of training data increases. Thus, a trained model is a monotone set function from the training dataset to reals, where the output is a valuation metric.

So far, we have discussed works that have used multiset to vector embeddings and analyzed them. We have taken special care to discuss works that have involved the structure of monotonicity and applied it to various domains. However, to the best of our knowledge, there has been no previous work that specifically deals with partial order preserving set-to-vector maps. One related work is~\cite{dushnik1941partially} which have studied order dimension of a partially ordered structure. Now, note that partial order preserving sets are injective as discussed before, and injectivity is in fact a weaker condition that the MAS property. Here, we discuss a number of works on injective functions on multisets.  The paper~\cite{amir_finite} shows that moments of neural networks
do define injective multiset functions, provided that an analytic non-polynomial
activation is used. They also state and
prove a finite witness theorem, which is of independent interest. There are a series of works such as~\cite{amir2024fourierslicedwassersteinembeddingmultisets, balan2022permutation, sverdlov} that design bi-Lipschitz(hence, injective) embeddings on sets and graphs. On the other hand, works like~\cite{wagstaff2019limitations, wang2023polynomial} investigate the impact of latent dimension on the expressive power of set functions such as DeepSets.

\newpage
\section{Proofs of the technical results}
\label{app:proofs}
\subsection{Proofs of existential results on \mas functions}
\finiteground*
\begin{proof}
\label{proof-of-1}
The construction of a monotone embedding with $n=|V|$ simply uses one-hot encoding. Namely, we identify $V$ with $[n]$.  For every $S\in \Pvinf$, we define
\begin{align}
    F(S) = \sum_{s\in S}e_s\in \mathbb{R}^{n}
\end{align}
where $e_s\in \RR^n$ is the vector with $e_s[s]=1$ and $e_s[j]=0$ for all $j\neq s$. We now show it satisfies all conditions:
Be $S\subseteq T$, then by definition, $\forall v\in V, c_S(v)\leq c_T(v)$, then $F(S) =\sum_{s\in S}e_s\leq \sum_{t \in T} e_t=F(T)$ thus $F(S)\leq F(T)$. On the other hand, if $F(S)\leq F(T)$ then $\forall v\in V, c_S(v)\leq c_T(v)$, and thus $S\subseteq T$.

For the second direction, assume there exists \mas $F:V \to \R^m$ for $m\leq n -1$.
For every output dimension $i \in [m]$, 
there is a "maximal singleton element" $v^*_i \in V$  such that $F(\set{v^*_i})[i] \geq F(\set{v})[i], \forall v \in V$.  The value of the set function over the union of such $v^*_i$-s across $i\in [m]$ will dominate any singleton that is disjoint from this collection, thanks to the monotonicity of $F$. As $m < |V|$, then $T:=\bigcup_{i=1}^m \set{v^*_i}$ does not cover $V$. Thus, we can select an element disjoint from this collection named $s\in V$, which gives disjoint sets $S=\{s\}, T$ with $F(S) \leq F(T)$, which contradicts separability. 
\end{proof}
\finitegroundsetrestrictedcardinality*
\begin{proof}
\label{proof_of_probablistic_construction}
For convenience of notation, in this proof we will assume without loss of generality that $V=[n]$. 
As discussed in the main text, the function $F_{a_1,\ldots,a_m}$ will be weakly monotone for any choice of non-negative vectors. Our goal is to show that there exists a set of parameters $a_1,\ldots,a_m\in [0,1]^n$, such that the obtained $F_{a_1,\ldots,a_m}$ is a monotone embedding, which means that, if $S \not\subseteq T$, then there exists some $j$ such that $h_j(S)>h_j(T)$, or in other words $\langle f(S),a_j \rangle>\langle f(T),a_j \rangle$. We note that to prove this, it is sufficient to consider pairs $(S, T) $ which we will call \emph{extreme pairs}. This means that $S$ consists of a single element  $x\in V$, and $T$ consists of $k$ elements in $V$ but does not contain $x$ (but repetitions are allowed). Indeed, if the claim is true for such pairs, and $S,T\in \Pvk $ such that $S \not\subset T$, then there exists an element $x$ which is in $S$ more times than it's in $T$. First, remove all copies of $x$ from $T$, and we know that at least one copy of $x$ will be in $S$, note that the difference between the function values remains the same. Then, we will 'replace' $S$ with a smaller set $\tilde S=\{x\}$, and $\tilde T$ from adding more elements to $T$ from $V \setminus \{x\}$ to obtain a new set $\tilde{T}$ of maximal cardinality $k$. We then will have that $(\tilde S,\tilde T)$ is an extreme pair, and so for some $j$ we will have  that $h_j(\tilde{S})> h_j(\tilde{T})$. Therefore, using the weak monotonicity of the function $h_j$, we will have 
$$h_j(S)\geq h_j(\tilde{S})> h_j(\tilde{T})\geq h_j(T).$$

Given an extreme pair $(S,T)$ we consider the set of all 'bad' $a$, namely
\begin{align*}
    B_{S,T}:=\{a\in [0,1]^n, \langle f(S),a \rangle \leq \langle f(T),a \rangle \}
\end{align*}

Our first goal will be to bound the Lebesgue measure of this set. Recall that $S$ is a singleton $S=\{s\}$, and $T$ does not contain this singleton. Accordingly,
\begin{align*}
B_{S,T}=\{a\in [0,1]^n| \quad a_{s}\leq \sum_{t \in T} a_t\}
\end{align*}
We will compute the probability of the complement of this set, namely vectors $a$ such that $a_s > \sum_{t \in T} a_t$. 
First denote by
\begin{align*}
    E_{S,T}=\{a\in [0,1]^n|\forall t\in T, a_t < \frac{a_s}{k}\}
\end{align*}
Note, that 
\begin{align*}
  E_{S,T} \subseteq \{a\in [0,1]^n| \quad a_{s} > \sum_{t \in T} a_t\} 
\end{align*}
So let's compute the measure of $E_{S,T}$:

Let $r=|\operatorname{supp}(T)|\leq k$.
\begin{align*}
    \mathrm{Leb}(E_{S,T})
    &= \mathrm{Leb}\left(\left\{a\in [0,1]^n \,\middle|\, \forall t\in T,\ a_t\leq \frac{a_s}{k}\right\}\right) \\
    &= \int_0^1 \left(\frac{a_s}{k}\right)^r da_s
    = \frac{1}{k^r(r+1)}
    \geq \frac{1}{k^k(k+1)}
    \geq \frac{1}{(k+1)^{k+1}},
\end{align*}
So, 
\begin{align*}
    \mathrm{Leb}(B_{S,T})\leq 1-\mathrm{Leb}(E_{S,T})  \leq 1 - \frac{1}{(k+1)^{k+1}}
\end{align*}

Now, consider the set 
$$B^m_{S,T}=\{a_1,\ldots,a_m \in B_{S,T} \},$$
Then we have 
$$ \mathrm{Leb}(B^m_{S,T})\leq\left(1 - \frac{1}{(k+1)^{k+1}}\right)^m$$
The vectors $a_1,\ldots,a_m $ will not define a monotone embedding, if there exists an extreme pair $S,T$  which is in $B^m_{S,T} $. The probability of this happening can be bounded by a union bound. There are $n\binom{n+k-2}{k}$ extreme pairs. Since
\[
n\binom{n+k-2}{k}\leq n(n-1)^k\leq n^{k+1}<n^{k+2}, 
\]
for simplicity we upper bound their number by $n^{k+2}$. The union bound will then give us 
$$\mathrm{Leb}\{a_1,\ldots,a_m \text{ do not define a monotone embedding } \}\leq n^{k+2} \left(1-\frac{1}{(k+1)^{k+1}}\right)^m  $$
It is sufficient to show that this expression is smaller than $1$, for our $m$, which, by taking a logarithm, is equivalent to requiring that 
$$(k+2)\ln \left( n \right)+ m\ln \left( 1-\frac{1}{(k+1)^{k+1}} \right)<0 $$
Using Taylor's expansion it can be shown that $\ln(1-x)<-x$ for $x\in (0,1)$, and as a result it is sufficient to choose $m$ so that
$$(k+2)\ln \left( n \right)-  \frac{m}{(k+1)^{k+1}} <0 $$
or equivalently 
$$m> (k+1)^{k+1}\cdot (k+2) \cdot  \ln(n). $$
For convenience we slightly enlarge this lower bound to obtain $m\geq (k+2)^{k+2} \ln(n) $
\end{proof}
\lowerBoundK*
\label{lower_bound1}
\begin{proof}
The proof is partially inspired by the proof in \cite{dushnik1941partially}.

\paragraph{Lower bound on $k$:}
We want to show the lower bound of $\mm(V,k) \geq \min\round{2k, n-2}$. As in the proof of \cref{thm:dim_geq_n}, we select a  "maximal singleton element" $v^*_i \in V$ for every dimension $i \in \set{1, 2, \cdots, \mm(V, k)}$. Thus we must have: \begin{align*}
    \text{for each }i \in [\mm(V, k)]: \ F(\set{v^*_i})[i] \geq F(\set{v})[i], \forall v \in V
\end{align*} 
Now, if $\mm(V, k) \leq |V|-2$,  we can find two elements $\set{u_1, u_2}$ disjoint from the collection $\mathcal{V} := \bigcup_{i=1}^{\mm(V,k)}\set{v^*_i}$.  Without loss of generality, we can assume that $F(\set{u_1})[i] \geq F(\set{u_2})[i]$ for most indices in $[\mm(V,k)]$. We can then construct a set $T$ containing $u_1$ and all the $v^*_i$ for all indices  $i$ for which $F(\set{u_1})[i] < F(\set{u_2})[i]$, thus we define:
\begin{align*}
    T := \set{u_1}\cup\set{v_j^*: F(\set{u_1})[j] < F(\set{u_2})[j]}
\end{align*}
We then have that:
\begin{align*}
    &F(\set{u_2})[i]\leq \max\{  F(\set{u_1})[i], F(\set{v^*_i})[i]  \} \leq F(T)[i], \forall i \in [\mm(V,k)]
\end{align*}
where the last inequality uses the monotonicity of $F$. This implies that $F(S) \leq F(T)$ However,  $S $ is not a subset of $T$. It follows that  $k \leq |T|-1 \leq \mm(V,k)/2$, as otherwise separability of $F$ in $\Pvk$ is violated. Hence, we get that: $\mm(V,k) \geq 2k$

\paragraph{Lower bound on $n$: }
For this proof, we will use the Erd\"os-Szekeres theorem, which in particular states, for a natural $N\geq 2$, that a sequence of real numbers with $\geq N^2$ elements has a subsequence of cardinality $N$ which is monotone (either monotonously increasing or monotonously decreasing). 
Assume we have $F:\Pvk\rightarrow \mathbb{R}^{m}$ a monotone embedding, where $V=[n]$. Assume by way of contradiction that: 
\begin{align*}
    \log_2(\log_3 n)>m  \text{ or equivalently }  n > 3^{2^m} 
\end{align*}

Consider the first coordinate and look at the sequence
\begin{align*}
     F(\{1\})[1],\ldots,F(\{n\})[1]
\end{align*}
This is a real-valued sequence, and so there is  a monotone subsequence $a_1<a_2<\ldots,<a_\ell $ of the original sequence $1,\ldots,n$, with  $\ell=\sqrt{3^{2^m}}=3^{2^{m-1}}$, such that:
\begin{align*}
    F(a_1)[1]\leq F(a_2)[1]\leq \ldots \leq F(a_\ell)[1] \text{ or } F(a_1)[1]\geq F(a_2)[1]\geq \ldots \geq F(a_\ell)[1]
\end{align*}
 Next, we consider the second coordinate of this subsequence, namely $F(a_1)[2],\ldots,F(a_\ell)[2]$ and obtain a new subsequence $\{a_j\}$ such that both the sequences: $\{F(a_j)[1]\}$ and  $\{F(a_j)[2]\}$ are ordered monotonically, and the size of the subsequence is the square root of the previous one, namely $3^{2^{m-2}}$.  After doing this  $m$ times, we have a subsequence of size $3$. Namely, we have three distinct elements $u,v,w \in V$ such that for all $i=1,\ldots,m$,
 \begin{align*}
 \text{either } F(\{u\})[i]\leq F(\{v\})[i] \leq F(\{w\})[i] \text{ or }  F(\{u\})[i]\geq F(\{v\})[i] \geq F(\{w\})[i]
 \end{align*}

It follows that for all $i$, we have:
\begin{align*}
    F(\{v\})[i]\leq \max\left(F(\{u\})[i],F(\{w\})[i] \right) \leq F(\{u,w\})[i]
\end{align*}
Thus, $F(\{v\}) \leq F(\{u,w\})$, which is a contradiction since $\{v\}$ is not contained in $\{u,w\}$. 
\end{proof}

\paragraph{A refined lower bound:}Now, we may extend the above proof idea to generalize the lower bound on $\mm(V,k)$, which is useful specially when $|V| >> k$. In such cases, the following gives a tighter lower bound on $\mm(V, k)$ which is quadratic in $k$.
\begin{lemma}
    \label{lemma:refined_bound_k}
    For all $\ell \in [k]$, we have $\mm(V, k) \geq \min\round{|V| - \ell, \ell k +2\ell - \ell^2}$. More specifically, if $\ell = \frac{k+2}2$, then $\mm(V, k) \geq \min\curly{|V| - \frac{k+2}2, \round{\frac{k+2}2}^2}$
\end{lemma}
\begin{proof}
    Consider any $\ell \in [k]$. Like in the proof just before, we consider a "maximal singleton element" $v^*_i \in V$ for every dimension $i \in \set{1, 2, \cdots, \mm(V, k)}$ and obtain the collection $\mathcal{V} := \bigcup_{i=1}^{\mm(V,k)}\set{v^*_i}$. Let us assume $\mm(V, k) \leq |V| - \ell$. Thus, we can produce the collection $\mathcal{U} := \set{u_1, u_2, \cdots, u_\ell}$ disjoint from $\mathcal{V}$, i.e $\mathcal{V} \cap \mathcal{U} = \emptyset$. For each $u_q \in \mathcal{U}$, let $m_q$ be the number of co-ordinates $j \in [\mm(V,k)]$ such that $F(\set{u_q})[j] > F(\set{u_r})[j], \forall r \in [\ell]\setminus\set{q}$. Clearly, $\sum_{q \in [\ell]}m_q \leq \mm(V, k)$, thus, $\exists \ell_0 \in [\ell]\setminus$ such that $m_{\ell_0} \leq \frac{\mm(V,k)}\ell$. Like before, we consider the sets $S = \set{u_{\ell_0}}$. We construct $T$ as follows: we select $\mathcal{U}\setminus\set{u_{\ell_0}}$ and take union with all $\set{v^*_j}$ for each dimension $j$ in which $F(\set{u_{\ell_0}})[j] > F(\set{u_r}), \forall r \in [\ell]\setminus\set{\ell_0}$. Since $\set{v^*_j}$ is the maximal singleton for the dimension $j$, it follows that $F(\set{u_{\ell_0}})[j] \leq F(\set{v^*_j})[j]$. Hence, if we define:
    \begin{align*}
        T = \round{\mathcal{U}\setminus\set{u_{\ell_0}}} \bigcup \set{v^*_j: F(\set{u_{\ell_0}})[j] > F(\set{u_r}), \forall r \in [\ell]\setminus\set{\ell_0}}
    \end{align*}
    By the choice of $\ell_0$, we have that $\abs{\set{v^*_j: F(\set{u_{\ell_0}})[j] > F(\set{u_r}), \forall r \in [\ell]\setminus\set{\ell_0}}} \leq \frac{\mm(V, k)}\ell$ Now, we must have that, $\forall j \in [\mm(V, k)], \exists t \in T\text{ such that }F(\set{t})[j] \geq F(\set{u_{\ell_0}})[j]$. By monotonicity of $F$, we thus have that $F(\set{u_{\ell_0}})[j] \leq F(T)[j], \forall j \in [\mm(V,k)]$ which implies $F(S) \leq F(T)$. But by design, $S \cap T = \emptyset$, which contradicts separability of $F$. Thus, we must have that, $|T| > k$, i.e $|T| - 1 \geq k$. But as stated before, $|T| \leq \frac{\mm(V,k)}\ell + \ell - 1$. Combining these two, we get that: $\frac{\mm(V,k)}\ell \geq k + 2 - \ell$, which implies $\mm(V, k) \geq k\ell + 2\ell - \ell^2$, thus proving our result.

    Now, we already have a sufficient upper bound on $\mm(V,k)$ from~\cref{thm:upper_bound_V} that's $O\round{k^{k+2}\log(n)}$. Now, when $n >> k$ we have that $O\round{k^{k+2}\log(n)} < n - O(k)$, in which case we may apply $\ell = \frac{k+2}2$ in the above lemma. And since we already know that $O\round{k^{k+2}\log(n)} < n - O(k)$, we thus get that $\mm(V,k) < n - O(k)$. Thus, in this case the above lemma gives us $\mm(V,k) \geq \round{\frac{k+2}2}^2$, which is a tigher quadratic lower bound in $k$.
\end{proof}
\paragraph{When $k=1$}\label{app:k=1}
In  the degenerate case where only a single set element is allowed, $k=1$, it is possible to construct a MAS function even for the uncountable ground set $V=[-1,1]$ by defining
$
F:\mathcal{P}_{\leq 1}([-1,1]) \rightarrow \mathbb{R}^{2}
$
via       
\begin{align}
	\label{degenerate_case}
	F(S) =
	(-1, -1) \ \text{if } S = \emptyset\quad  \text{ and }\quad  F(S) =   
	(-x, x)  \  \text{if } S = \{x\}, 
\end{align}
\paragraph{Separable embeddings (even non monotone) don't exist:}Now, as discussed in the main text, we give a justification here on why we choose to relax separability and not monotonicity in the definition of wekaly \mas embeddings. Here, we show that, with the added assumptions of injectivity and continuity, even non-monotone set-t-vector embeddings do not exist in any dimension for uncountable ground sets.
\begin{theorem}
\label{thm:separable_general}
    There does not exist a continuous injective separable set function taking values in $\R^m$ for any $m$ when the ground set $V$ is an open subset of $\R^d$
\end{theorem}
\begin{proof}
    \label{Impossibility_separating}
    We restrict ourselves to $\mathcal{P}_{=n}(V)$, i.e, only subsets of size $n$. Since $F$ is continuous and permutation-invariant, according to Theorem-7 of \cite{zaheer2018deepsets}, $f(\{x_1, \cdots, x_n\}) = \rho(\sum_{i=1}^n \phi(x_i))$, where $\rho, \phi$ are continuous and bijective(given in condition). Now, consider the space $\St \subset [0,1]^n, \St = \{(x_1, \cdots, x_n): x_1 < x_2 < x_3 < \cdots < x_n;  0 < x_i < 1, \forall i \in [n]\}$. Note that $\St$ is open in $[0,1]^n$ (to see this, let $\epsilon = \min_{i \in [n-1]} (x_{i+1} - x_i)$. The ball of radius $\frac{\epsilon}{2}$ is contained in the domain). Since $ F: \St \to \R^n$ is an injective, continuous map. Thus, by Invariance of Domain Theorem(\cite{munkres198479}), $F(\St)$ is open in $\R^n$ as well. Note that each element of $S$ uniquely corresponds to an element of $\mathcal{P}_{=n}(V)$. Thus, $F(\mathcal{P}_{=n}(V))$ is open in $\R^n$. Now, consider any set $A = \{a_1, \cdots, a_n\}$. Since $F$ is an open map, we have that $\exists \delta > 0$ such that $B(F(A), \delta) \subseteq F(\mathcal{P}_{=n}(V))$. Thus, $\exists B \in \mathcal{P}_{=n}(V)$ such that $F(B)_i = F(A)_i + \frac{\delta}{2n}, \forall i \in [n]$. But then, $F(B) > F(A)$ implies $A \subset B$ but $|A| = |B| = n$ by construction. This gives a contradiction.
\end{proof}

\subsection{Proofs of results on on weakly \mas functions}

\paragraph {Proposition 6.}
   If the deep set model in \eqref{eq:parametric_setfn} is implemented with a vector-to-vector monotonously increasing function $\mtwo$ and a non-negative activation $\sigma$, then $F(\bullet; (A,b))$ is monotone for every $(A,b)$. If in addition $\sigma$ is monotone (increasing or decreasing), then there exists $S \not \subseteq T$ with $F(S; (A,b))\leq F(T; (A,b)) $ for all $A,b$.   
\begin{proof}
 Let $x<y$ and consider the two sets: 
 \begin{align*}
     S = \{\frac{x+y}{2}\},T = \{x,y\}
 \end{align*}
 And we claim that for all $A,b$, although $S\not\subseteq T$, $F(S,(A,b))\leq F(T,(A,b))$. Note that as $x<\frac{x+y}{2}<y$, then it must have been that 
 \begin{align*}
     Ax+b\leq A\cdot\frac{x+y}{2}+b\leq Ay+b 
 \end{align*}
 Or
  \begin{align*}
     Ay+b\leq A\cdot\frac{x+y}{2}+b\leq Ax+b 
 \end{align*}
 Then, by (weakly) monotonicity of the activation $\sigma$ and of $M_2$, it must be that 
 \begin{align*}
     \sigma(A\cdot\frac{x+y}{2}+b)\leq \sigma(A\cdot x+b)
 \end{align*}
 Or that 
  \begin{align*}
     \sigma(A\cdot\frac{x+y}{2}+b)\leq \sigma(A\cdot y+b)
 \end{align*}
 In any case, it's true that 
 \begin{align*}
     \sigma(A\cdot\frac{x+y}{2})\leq \sigma(A\cdot x + b)+\sigma(A\cdot y+b)
 \end{align*}
 Then, by monotonicity of $\mtwo$, 
 \begin{align*}
          F(S,(A,b))=\mtwo\cdot(\sigma(A\cdot\frac{x+y}{2}+b))\leq \mtwo\cdot(\sigma(A\cdot x+b)+ \sigma(A\cdot y+b))=F(T,(A,b))
 \end{align*}
 Thus, 
 \begin{align*}
     F(S,(A,b))\leq F(T,(A,b))
 \end{align*}
\end{proof}

\STnonmon*
\begin{proof}
We need to cone up with a class of tuple of parameters $(W_Q, W_K, W_V)$ such that the function $S \mapsto F(S; W_Q, W_K, W_V)$ is not monotone. Choose $W_Q = W_K = W$ such that $W$ is full-rank. Choose any non-zero $W_V$. Now, choose some vector $\bx_1$ such that $W_V\bx_1$ is non-zero. By swapping $\bx_1$ with $-\bx_1$ if necessary, we can choose some index $j$ such that for $\bv_1 = W_V\bx_1$, the $j$-th index is negative, \ie $\bv_1[j] < 0$. By assumption we also have that $\bq_1 = \bk_1$ is a non-zero vector, and thus has a positive norm. Now let $S$ be the singleton set $S = \set{\bx_1}$, which is a subset of $T = \set{\bx_1, \mathbf{0}_d}$ where $\mathbf{0}_d$ is the $d-$dimensional all-zero vector. Note that, $F(S; W_Q, W_K, W_V) = \bv_1$. And since $\bq_2, \bk_2, \bv_2$ are all zero, we have: $F(T; W_Q, W_K, W_V) = (\alpha_{1,1} + \alpha_{1,2})\bv_1$. Note that, $\frac{\alpha_{1,1}}{\alpha_{1,1}+\alpha_{1,2}} = \frac{e^{\norm{\bq_1}^2}}{e^{\norm{\bq_1}^2} + 1} + \frac12 > 1$. Hence, it follows that $F(T; W_Q, W_K, W_V)[j] < F(S; W_Q, W_K, W_V)[j]$ and thus, the non-monotonicity follows
\end{proof}

\paragraph{Proposition 8.}
    If $\mtwo$ is strictly monotone increasing, and $\sigma$ is a hat activation function, then $F(S,(a,b))=\mtwo(\sum_{x\in S}\sigma(a\tran x+b)) $ is weakly \mas function. 
\begin{proof}
    It's clear the function is a set function.
    Be $S\subseteq T$, as $\sigma$ is positive, then for any $a,b$, 
    \begin{align*}
        \sum_{x\in S}\sigma(a^Tx+b)\leq \sum_{x\in S}\sigma(a^Tx+b) + \sum_{x\in T\setminus S}\sigma(a^Tx+b)=\sum_{x\in T}\sigma(a^Tx+b)
    \end{align*}
    So 
    \begin{align*}
        F(S,(a,b))\leq F(T,(a,b))
    \end{align*}
    Denote by $[s_1,s_2]$ the support of $\sigma$,and choose some point in it such that $s\in [s_1,s_2]:\sigma(s)>0$.
    Now, be $S\not\subseteq T$, then there exists $z\in V:C_S(z)>C_T(z)$.
    Now, remove $z$ from both sets $C_T(z)$ times. This subtracts from both function values the same amount.
    Now, we have that $z\in S,z\not\in T$. Denote by $\epsilon =\frac{min_{t \in T}|z-t|>0}{2}$. Choose
    \begin{align*}
        a = \frac{s_2-s_1}{\epsilon}, b =-(a\cdot z - s)
    \end{align*}
    We claim that $\forall t\in T, \sigma(a\cdot t+b)=0$. Assume it's not zero, thus $s_1 < a\cdot t + b< s_2$. Thus
    \begin{align*}
         s_1 < \frac{s_2-s_1}{\epsilon}\cdot t - a\cdot z + s  < s_2
    \end{align*}
    Thus,
    \begin{align*}
        s_1-s < \frac{s_2-s_1}{\epsilon}\cdot t -\frac{s_2-s_1}{\epsilon}\cdot z  < s_2-s
    \end{align*}
    Thus, 
    \begin{align*}
        \frac{s_1-s}{s_2-s_1}\cdot \epsilon <  t - z  < \epsilon \cdot \frac{s_2-s}{s_2-s_1}
    \end{align*}
    But, 
    \begin{align*}
        |\frac{s_1-s}{s_2-s_1}|\leq 1 \\
        |\frac{s_2-s}{s_2-s_1}|\leq 1
    \end{align*}
    Thus, 
    \begin{align*}
        -\epsilon\leq t - z \leq \epsilon
    \end{align*}
    A contradiction to the definition of $\epsilon$. Thus, $F(T,(a,b))=\mtwo(0)$.
    But 
    \begin{align*}
         \sigma(a\cdot z + b) = \sigma(a\cdot z-a\cdot z +s) = \sigma(s)>0
    \end{align*}
    Thus 
    \begin{align*}
        F(S,(a,b)) = \mtwo(\sum_{x\in S}\sigma(a\tran x+b))\geq \mtwo(\sigma(s_1))>\mtwo(0)=F(T,(a,b))
    \end{align*}
    And we are done.
\end{proof}

\proprelumas*
\begin{proof}
    Consider the function $\tri$ as in the third plot of Fig.~\ref{fig:activations}. We know that $\tri$ belongs to the class of \hhat functions. From \ref{prop:weakmas_hat} we know that, for any $S \not \subseteq T \in \Pvinf$, $\exists \ba \in \R^d$ and $b \in \R$ such that $\sum_{x \in S} \tri(\ba\tran x + b) > \sum_{y \in T} \tri(\ba\tran y + b)$. Now, consider the following parameters:
    \begin{equation}
        \label{eq:relu_mas_params}
        \begin{split}
            & \boldA = \sbrac{\ba, \ba, \ba}\tran \in \R^{3 \times d}, \boldb = [b+1, b-1, b]\tran \in \R^3\\
            & \ba_1 = [1,1,-2]\tran \in \R^3, \text{ and }b_1 = 0
        \end{split}
    \end{equation}
   Then, for the above values of parameters, we have: 
   \begin{align*}
       \relu \circ \mathcal{A}_2(x) = \sbrac{\relu\round{(\ba\tran x + b) + 1}, \relu\round{(\ba\tran x + b) - 1}, \relu\round{\ba\tran x + b}}
   \end{align*}
   Thus, we have: 
   \begin{align*}
       &\ba_1\tran\round{\relu \circ \mathcal{A}_2(x)} + b_1 \\
       & =  \relu\round{(\ba\tran x + b) + 1} +  \relu\round{(\ba\tran x + b) - 1}- 2 \relu\round{\ba\tran x + b} \\
       &= \tri\round{\ba\tran x + b}
   \end{align*}
   Thus we have, $\relu \circ \mathcal{A}_1 \circ \relu \circ \mathcal{A}_2(x) = \relu \circ \tri\round{\ba\tran x + b} = \tri\round{\ba\tran x + b}$. Hence, $F(S) = \sum_{x \in S} \tri\round{\ba\tran x + b} > \sum_{y \in S} \tri\round{\ba\tran y + b} = F(T)$ and weak separability is proved. Since $\mtwo$ is monotone increasing and $\relu$ is non-negative, it follows that $F$ is monotone as well, hence weakly \mas.
\end{proof}
\subsection{H\"older separability of \mas functions}
\flowerholder*

\begin{proof}
    We consider $V \subset \R^d$ to be a compact set with maximum norm of $1$. Also, we note that, if $\sigma(x)$ is a \hhat function with support in $[\gamma_1, \gamma_2]$, then $\sigma(\frac{x - \gamma_1}{\gamma_2 - \gamma_1})$ is a \hhat function with support in $(0,1)$. For this proof, we thus consider $\sigma$ to have support in $[0,1]$. For support in general $[\gamma_1, \gamma_2]$ and a general norm bound of $V$, the distributions on $b, c$ needs to be scaled and shifted to get same Lower H\"older results with the H\"older constants scaled appropriately. \footnote{If $\sup_{x \in V}\norm{x} \leq B$, and $\supp(\sigma)\subseteq[\gamma_1, \gamma_2]$ then distributions of $b, c$ should be shifted by $\gamma_1$ and linearly scaled by $\delta: = \frac{B}{\gamma_2 - \gamma_1}$ and to be: $b \sim \Unif(-\gamma_1-\delta, -\gamma_1+\delta)$ and $c \sim \Unif(0, 2\delta)$ respectively}. We recall the definition of $\das(S, T)$ from equation~\ref{eq:asymmetric_dist}. Let $S = \set{x_1, \cdots, x_\sizes}$ and let $T = \set{y_1, \cdots, y_\sizet}$, where $S, T \subset V$. Clearly by the definition of $\das(S, T)$, it is sufficient to consider $|S| \leq |T|$, i.e $\sizes \leq \sizet \leq k$.
    
    Using the definition of EMD gives us the following equivalent definition of $\das(S, T)$: If $\Omega_{\sizes, \sizet} := \set{ \tau : [\sizes] \to [\sizet], \tau \text{ is injective}}$ , then $\das(S, T) = \min_{\tau \in \Omega_{\sizes, \sizet}}\sum_{i=1}^\sizes\norm{x_i - y_{\tau(i)}}_2$. Let $\tau^*$ be the optimal coupling between $[\sizes] \text{ and } [\sizet]$, i.e $\tau^* = \arg\min_{\tau \in \Omega_{\sizes, \sizet}}\sum_{i=1}^\sizes\norm{x_i - y_{\tau(i)}}_2$, which means $\das(S, T):=\sum_{i=1}^\sizes\norm{x_i - y_{\tau^*(i)}}_2$. Now, $\forall i \in [\sizes]$, we define $\Delta_i := \norm{x_i - y_{\tau^*(i)}}_2$. For any $x \in \R^d, r > 0$, let $\B(x, r)$ denote the Euclidean ball centered at $x$ with radius $r$. We define $\Omega_i := \set{\ell \in [\sizes]: y_{\tau^*(\ell)} \in \B(x_i, \Delta_i)}, \forall i \in [M]$. With these notations, we make the following observation:

    \paragraph{Claim-1:}
        Consider any $ r \in (0, k)$. Then, $\exists i_0 \in [\sizes]$ such that $\sum_{\ell \in \Omega_{i_0}}\Delta_\ell \leq \Delta_{i_0}r$ and $\Delta_{i_0} \geq \frac{r^k}{k^{k+1}}\das(S, T)$ 
    
    \begin{proof}
        For the sake of contradiction, we assume otherwise, i.e $\forall i \in [\sizes]$, we have: either $\Delta_i \leq \frac{r^k}{k^{k+1}}\das(S, T)$ or $\sum_{\ell \in \Omega_i}\Delta_\ell \geq \Delta_i r$. WLOG, assume the following order on $\Delta_i$'s: $\Delta_1 \geq \Delta_2 \geq \cdots \geq \Delta_M$. We build the following directed graph $\Gcal$ on $[\sizes]$. Start with the node $v_0 = 1$ and add we add one node to $\Gcal$ at a time. At step $t$, we select $v_t := \arg\max_{u \in \Omega_{v_{t-1}}} \Delta_u$ and add the edge $(v_{t-1}, v_t)$ to $\Gcal$ along with the new node $v_t$. 
        
        Note that, by definition of $\das(S, T)$ we have that: $\das(S, T) = \sum_{i=1}^{\sizes}\Delta_i$. This, implies that: $\Delta_{v_0} \geq \frac{\das(S, T)}M$, as we take argmax over entire $[\sizes]$ .Thus, $\Delta_1 \geq \das(S,T)/M \geq \das(S,T)/k$. Thus, by the contradictory assumption, we must have: $\sum_{u \in \Omega_{v_0}}\Delta_u \geq \Delta_1 r$. Since $v_1 = \arg\max_{u \in \Omega_{v_0}}\Delta_u$, it thus follows that: $\Delta_{v_1}\geq \frac rk\Delta_1 \geq \frac {r}{k^2}\das(S,T)$. Following the similar argument, we thus have, $\Delta_{v_2} \geq \frac rk)\Delta_{v_1}$ and so on. Thus, $\forall t \in \NN$ we must have: $\sum_{u \in \Omega_{v_t}}\Delta_u \geq \Delta_{v_t}r$. Thus, for $v_{t+1}$, defined by: $\arg\max_{u \in \Omega_{v_t}}\Delta_u$, we must have: $\Delta_{v_{t+1}} \geq \frac rk \Delta_{v_t}$. Since this holds for every $t \in \NN$ and since $r > 0$ by hypothesis, we must have: $\Delta_{v_t} \geq \frac rk \Delta_{v_{t-1}} \geq \round{\frac rk}^2 \Delta_{v_{t-1}} \geq \cdots \geq \round{\frac rk}^t\Delta_{v_0}$. 
        
        As we work with finite $\sizes, \sizet$ either $\Gcal$ is a DAG and this process terminates; or $\Gcal$ has a cycle. In former case, at termination step $\bm{t_t}$ we cannot add any more nodes, thus $\Omega_{\bm{t_t}} = \emptyset$. Since all nodes are unique in this case, we must have $\bm{t_t} \leq k$. So, $\Delta_{\bm{t_t}} \geq \round{\frac rk}^{\bm{t_t}}\Delta_{v_0} \geq \frac1k \round{\frac rk}^{\bm{t_t}}\das(S, T)$. Since $\frac rk \in (0,1)$ by hypothesis and since $\bm{t_t} \leq k$, we thus have: $\Delta_{\bm{t_t}} \geq \frac 1k \round{\frac rk}^k\das(S, T)$, and hence done with this case.
        
        For the second case, consider the smallest cycle in $\Gcal$. Let this cycle of length $p$ be $\mathcal{C} = \set{v_\beta \to v_{\beta+1} \to \cdots \to v_{\beta + {p-1}} \to v_\beta}$ for some $\beta \geq 0$. Then, we have: $\beta + i + 1 \in \Omega_{\beta+i}, \forall i \in \set{0, \cdots p-1}$, where addition of indices is modulo $p$. Thus we have: $\norm{x_{v_\beta + i} - y_{\tau^*(v_\beta + i +1)}}_2 < \norm{x_{v_\beta + i} - y_{\tau^*(v_\beta + i)}}_2, \forall i \in \set{0, \cdots, p-1}$. Taking sum, we have the following: $\norm{x_{v_\beta} - y_{\tau^*(v_\beta+1)}}_2 + \norm{x_{v_\beta+1} - y_{\tau^*(v_\beta+2)}}_2 + \cdots + \norm{x_{v_\beta + p-1} - y_{\tau^*(v_\beta)}}_2 < \sum_{j=0}^{p-1}\norm{x_{v_\beta+j} - y_{\tau^*(v_\beta+j)}}_2$. Hence, switching from the coupling $i \mapsto \tau^*(i)$ to $i \mapsto \tau^*(i+1)$ for $i \in \set{v_\beta, \cdots, v_{\beta + p-2}}$ and from $v_\beta + p - 1 \mapsto \tau^*(v_\beta+ p - 1)$ to $v_\beta + p - 1 \mapsto \tau^*(v_\beta)$ strictly reduces the value of the sum $\sum_{i=1}^\sizes\norm{x_i - y_{\tau^*(i)}}_2$, but we started with the optimal coupling $\tau^*$. This gives the reqd. contradiction. 
    \end{proof}
    We also make the following observation regarding this "optimal coupling" $\tau^*$:
    
    \paragraph{Observation-2:} For any given $i \in [\sizes]$, if $y_j \in T \cap \Omega_i$ for some $j \in [\sizet]$, then $\exists \ell \in [\sizes]$ such that $j = \tau^*(\ell)$
    \begin{proof}
        The proof follows from the following observation: if $y_j \in \Omega_i$ for any $i \in [\sizes]$, then we must have: $\norm{x_i - y_j}_2 \leq \Delta_i = \norm{x_i - y_{\tau(i)}}_2$. Thus, if $y_j$ was "free", i.e if $\not \exists \ell \in [\sizes]: j = \tau^*(\ell)$, then we can switch from $i \mapsto \tau^*(i)$ to $i \mapsto j$, and this would reduce the sum $\sum_{i=1}^\sizes\norm{x_i - y_{\tau^*(i)}}_2$, but $\tau^*$ was the optimal coupling. This gives a contradiction
    \end{proof}

    Now, as per the starting discussion, we consider $\supp(\sigma) = [0,1]$, and the proof for general support of $[\gamma_1, \gamma_2]$ requires the scaling and shifting as mentioned before.
    
    By assumption we have, $\lim_{ t\to 0^+}\sigma\dash(t) > 0$. We define $\kappa_1 := \frac{\lim_{ t\to 0^+}\sigma\dash(t)}2$. Since $\sigma$ is piecewise continuously differentiable as per assumption, we must have $\sigma\dash(t) \geq \kappa_1, \forall t \in (0, \omega)$ for some $\omega \in (0,1)$ by the continuity of $\sigma\dash$. Note that, by Lagrange's Mean Value Theorem, we have that:
    \begingroup
    \allowdisplaybreaks
    \begin{align*}
        &\forall x, y \in (0, \omega), \frac{\abs{\sigma(x) - \sigma(y)}}{\abs{x-y}} = \sigma\dash(z) , \text{ where } z \in (x,y)\\
        &\implies \forall x, y \in (0, \omega), \abs{\sigma(x) - \sigma(y)} \geq \kappa_1\abs{x-y}, \text{ since } z \in (0, \omega) \\
        &\implies \lim_{y \to 0^+} \abs{\sigma(x) - \sigma(y)} \geq \kappa_1\abs{x}, \text{ as }\sigma\text{ is cts.}\\
        &\implies \sigma(x) \geq \kappa_1x, \forall x \in (0, \omega), \text{ as }\sigma \geq 0
    \end{align*}
    \endgroup
    
     Also, $\sigma$ is upper Lipschitz with constant $\kappa_2 > 0$ per definition of Hat fn.  Now, consider the $i_0$ coming from the Claim as proved above. We define $\omega':= \min(\frac{\omega}{1-\omega}, 1)$. Let $P \in \NN$ be such that $\frac {\kappa_1}{P} < \frac 12$ and $\frac{\kappa_1}{\kappa_2 P} < k$ and $\omega' > \frac 2P$. Note that, such a $P \in \NN$ exists, since taking $P \to \infty$ satisfies all the 3 conditions. 
     
     Now, for a given $a \in \S^{d-1}$, consider the sets $S_a = \set{a\tran x: x \in S}, T_a = \set{a\tran y: y \in T}$.  Note that, since we have $\norm{x_i}, \norm{y_j} \leq 1, \forall i \in [M], j \in [N]$ as per earlier assumption, we get $S_a, T_a \subseteq (-1,1)$. For given $a \in \S^{d-1}$, we define the "optimal coupling" between the real-valued sets $S_a, T_a$ as: $\tau^*_a:= \arg\min_{\tau \in \Omega_{M, N}} \sum_{i=1}^M\abs{a\tran x_i - a\tran y_{\tau(i)}}$. Based on this, we define $\Delta^a_i := \abs{a\tran x_i - a\tran y_{\tau^*_a(i)}}$ and $\Omega^a_i = \set{\ell \in [M]: y_{\tau^*_a(\ell)} \in \B(a\tran x_i, \Delta^a_i)}$. We have: $\das(S_a, T_a) := \sum_{i=1}^M \Delta^a_i$. Let $i_0 \in [M]$ be the index that comes from applying the claim to the sets $S_a, T_a$. Now, we define the regions $\bm{B}_a := (a\tran x_{i_0} - \omega'\Delta^a_{i_0}, a\tran x_{i_0} - \frac{\Delta^a_{i_0}}P)$ and for a given pair $(a \in \S^{d-1}, b \in \R)$, we consider $\bm{C}_{a,b} := (\frac{a\tran x_{i_0}-b}\omega, a\tran x_{i_0}+\Delta^a_{i_0}-b)$. Thus, we have:
     \begin{align*}
         &\round{a\tran x_{i_0}+\Delta^a_{i_0}-b} - \round{\frac{a\tran x_{i_0}-b}\omega} \\
         = & \frac{(1-\omega)(b - a\tran x_{i_0}) + \omega\Delta^a_{i_0}}{\omega}\\
         = & \frac{1-\omega}\omega\round{(b- a\tran x_{i_0}) + \frac{\omega}{1-\omega}\Delta^a_{i_0}} \\
         \geq & \frac{1-\omega}\omega\round{b- (a\tran x_{i_0} - \omega'\Delta^a_{i_0})} > 0, \forall b > a\tran x_{i_0} - \omega'\Delta^a_{i_0}
     \end{align*}
    The above analysis shows that $\bm{C}_{a, b} \neq \emptyset, \forall b \in \bm{B}_a$, so the intervals are well defined. Now, if $b \in \bm{B}_a$ and $y < a\tran x_{i_0} - \Delta^a_{i_0}$ then we have: $y - b < a\tran x_{i_0} - \Delta^a_{i_0} - a\tran x_{i_0} + \omega'\Delta^a_{i_0} = (\omega' - 1)\Delta^a_{i_0} \leq 0$ by definition of $\omega'$. On the other hand, if $b \in \bm{B}_a$ and $ c \in \bm{C}_{a,b}$, for given any $ y > a\tran x_{i_0} + \Delta^a_{i_0}$ we have: $\frac{y - b}c \geq \frac{a\tran x_{i_0} - b + \Delta^a_{i_0}}{a\tran x_{i_0} - b + \Delta^a_{i_0}} =1$. Combining these, we get that if $\abs{y - a\tran x_{i_0}} \geq \Delta^a_{i_0}$, then $\frac{y - b}c \not \in (0,1), \forall b \in \bm{B}_a, c \in \bm{C}_{a,b}$. Thus, conditioning on a given $a \in \S^{d-1}$, we can write the conditional expectation taken over $b, c$ as follows:
    \begingroup 
    \allowdisplaybreaks
    \begin{align*}
        & \Exp_{b, c}\sbrac{\sum_{x \in S}\sigma\round{\frac{a\tran x-b}c}- \sum_{y \in T}\sigma\round{\frac{a\tran y-b}c}}_+\\
        \geq & \Exp_{b, c}\sbrac{\round{\sum_{x \in S}\sigma\round{\frac{a\tran x-b}c}- \sum_{y \in T}\sigma\round{\frac{a\tran y-b}c}}\indic{b \in \bm{B}_a, c \in \bm{C}_{a,b}}}_+, \text{using non-negative RV}\\
        \geq & \frac 14 \int_{b=a\tran x_{i_0} - \omega'\Delta^a_{i_0}}^{a\tran x_{i_0}}\int_{c=\frac{a\tran x_{i_0} - b}\omega}^{a\tran x_{i_0} + \Delta^a_{i_0} - b} \bigg[\sigma\big(\frac{a\tran x_{i_0}-b}c\big) - \\
        &\sum_{\ell \in \Omega^a_{i_0}} \bigg|\sigma\big(\frac{a\tran y_{\tau^*_a(\ell)}-b}c\big) \ - \sigma\big(\frac{a\tran x_\ell-b}c\big)\bigg|\bigg]_+\dif c\dif b\\
        \geq &\frac 14 \int_{b=a\tran x_{i_0} - \omega'\Delta^a_{i_0}}^{a\tran x_{i_0}}\int_{c=\frac{a\tran x_{i_0} - b}\omega}^{a\tran x_{i_0} + \Delta^a_{i_0} - b} \sbrac{\kappa_1 \frac{a\tran x_{i_0} - b}{c} - \kappa_2 \frac{\kappa_1}{P\kappa_2}\frac{\Delta^a_{i_0}}c }_+\dif c\dif b, \text{from Lipschitz cts. }\sigma\\
        =&\frac{\kappa_1}4\int_{b=a\tran x_{i_0} - \omega'\Delta^a_{i_0}}^{a\tran x_{i_0}}\sbrac{a\tran x_{i_0} - b - \frac{\Delta^a_{i_0}}P}_+\int_{c=\frac{a\tran x_{i_0} - b}\omega}^{a\tran x_{i_0} + \Delta^a_{i_0} - b}\frac{\dif c}{c} \dif b, \text{as }\relu(rx)=r\relu(x)\\
        =&\frac{\kappa_1}4\int_{b=a\tran x_{i_0} - \omega'\Delta^a_{i_0}}^{a\tran x_{i_0}-\frac{\Delta^a_{i_0}}P}\round{a\tran x_{i_0} - b - \frac{\Delta^a_{i_0}}P}\ln\round{\frac{\omega(a\tran x_{i_0}+\Delta^a_{i_0}-b)}{a\tran x_{i_0}-b}} \dif b\\
        =&\frac{\kappa_1}4\int_{z=\frac{\Delta^a_{i_0}}P}^{\omega'\Delta^a_{i_0}}\round{z - \frac{\Delta^a_{i_0}}P}\ln\round{\frac{\omega(z+\Delta^a_{i_0})}z}\dif z, \text{putting }z:=a\tran x_{i_0}-b\\
        \geq & \frac{\kappa_1}4\int_{z=\frac{\Delta^a_{i_0}}P}^{\frac{\omega'\Delta^a_{i_0}}2}\round{z - \frac{\Delta^a_{i_0}}P}\ln\round{\omega\round{1+\frac{\Delta^a_{i_0}}z}} \dif z. \text{using non-negativity of integrand}\\
        \geq & \frac{\kappa_1}4\ln\round{\omega\round{1 + \frac{2}{\omega'}}}\int_{z=\frac{\Delta^a_{i_0}}P}^{\frac{\omega'\Delta^a_{i_0}}2}\round{z - \frac{\Delta^a_{i_0}}P}\dif z\\
        \geq & \frac{\kappa_1}8\ln\round{2-\omega}\round{\frac{\omega'}2-\frac 1P}^2(\Delta^a_{i_0})^2, \text{ using } \omega' \leq \frac{\omega}{1-\omega}\\
        \geq& C \cdot \frac 1k \round{\frac{\kappa_1}{\kappa_2 P}}^k (\das^a(S, T))^2, \text{ where }C> 0\text{ as }\omega \in (0,1) \text{ and } P \geq \frac{2}{\omega'}\text{ by earlier choice}
    \end{align*}
    \endgroup
   Thus, we have shown that: $\Exp_{b, c}\bigg[\big[F(S)- F(T)\big]_+\big| a \bigg] \geq C'\cdot(\das^a(S,T))^2$. By law of total expectation, we have:
   \begingroup
   \allowdisplaybreaks
   \begin{align*}
       &\Exp_{a, b, c}\bigg[F(S) - F(T)\bigg]_+ = \Exp_a \Exp_{b, c}\bigg[\big[F(S)- F(T)\big]_+\big| a \bigg]\\
       \geq & C'\Exp_a \big(\das^a(S, T)\big)^2 \geq C'\round{\Exp_a \das^a(S,T)}^2, \text{by Jensen's ineq.}
   \end{align*}
   \endgroup
   Thus, to show Lower H\"older separability with exponent $\lambda =2$, it remains to show that, $\Exp_{a \sim \Unif(\S^{d-1})} \das^a(S,T) \geq c_1 \cdot \das(S, T)$, where $c_1 > 0$ is a constant. Now, we make the following observation:

   \paragraph{Observation-3:}Suppose we're given given $\ell \in \NN$ and $\ell$ many non-zero vectors $u_1, u_2, \cdots, u_\ell \in \R^d$. If $ a \sim \S^{d-1}$, then $\exists \delta(\ell) > 0$ such that: $\Prob\curly{\abs{a\tran u_i} \geq \delta(\ell) \norm{u_i}, \forall i \in [\ell]} \geq \frac 12$
   \begin{proof}
       Firstly, we note that, for a fixed $i \in [\ell]$, $\abs{a\tran u_i}$ is a continuous non-negative real-valued random variable. Thus, $\exists t_i \in \R_+$ such that $\Prob\curly{\abs{a\tran u_i} \geq t_i} \geq \frac 12$. Now, Choosing $\delta_i := \frac{t_i}{\norm{u_i}} > 0$ gives us: $\Prob\curly{\abs{a\tran u_i} \geq \delta_i\norm{u_i}} \geq \frac 12, \forall i \in [\ell]$. Since $\delta_i > 0, \forall i \in [\ell]$ and $\ell$ is finite, we thus have: $\delta:= \min_{i \in \ell}\delta_i > 0$. For this particular choice of $\delta$, we have that $\Prob\curly{\abs{a\tran u_i} \geq \delta\norm{u_i}} \geq \frac 12, \forall i \in [\ell]$, and the observation is proved.
   \end{proof}.
   Now, consider all $\bm{\ell_0} := \binom{N}{M}M!$ vectors $(x_i - y_{\tau(i)}) \in \R^d$, where $\tau$ runs over all injective functions from $[M] \to [N]$. For this particular choice of $\bm{\ell_0}$, we get from the last observation that, $\exists \delta(N, M) > 0$ such that $\Prob\curly{\abs{a\tran\round{x_i - y_{\tau(i)}}} \geq \delta(N, M)\norm{x_i - y_{\tau(i)}}} \geq \frac 12$, for all possible injective $\tau: [M]  \to [N]$, we call this set of $a \in \S^{d-1}$ to be $A \subseteq \S^{d-1}$. Now, given a fixed $a \in \S^{d-1}$, we have a particular $\tau_a: [M] \to [N]$ that gives $\das^a(S, T) = \sum_{i=2}^M\abs{a\tran\round{x_i - y_{\tau_a(i)}}}$. Also, let $\tau^* : [M] \to [N]$ be the "optimal coupling" such that $\das(S, T) = \sum_{i=1}^M\norm{x_i - y_{\tau^*(i)}}$. Now, we can write:
   \begin{align*}
       & \Exp_{a \sim \Unif(\S^{d-1})} \das^a(S,T)\\
       = & \Exp_{a \sim \Unif(\S^{d-1})}\sbrac{\sum_{i=1}^M\abs{a\tran(x_i - y_{\tau_a(i)})}}\\
       \geq & \Exp_{a \sim \Unif(\S^{d-1})}\sbrac{\sum_{i=1}^M\abs{a\tran(x_i - y_{\tau_a(i)})}\bigg|a \in A}\Prob_a\curly{a \in A} \\
       \geq & \frac{\delta(N, M)}2 \Exp_a\sbrac{\sum_{i=1}^M \norm{x_i - y_{\tau_a(i)}}_2} \\
       \geq & \frac{\delta(N, M)}2 \Exp_a\sbrac{\sum_{i=1}^M \norm{x_i - y_{\tau^*(i)}}_2}, \text{as }\tau^*\text{ is optimal coupling}\\
       =& \frac{\delta(N, M)}2\round{\sum_{i=1}^M \norm{x_i - y_{\tau^*(i)}}_2} = \frac{\delta(N, M)}2\das(S, T)
   \end{align*}
   Hence, $F$ is H\"older separable with constant $\frac{C'\delta(N,M)}2$ and exponent $\lambda = 2$.
\end{proof}
\probabilityseparation*
\begin{proof}
    Denote by $f$ the scalar-valued set function for a co-ordinate of $F$; we proved it's a lower Holder function.
    Let $S\not\subseteq T$, then by lower continuity we know that,
    \begin{equation}
        \Exp_{w \sim \mu(\cdot)}\norm{\sbrac{f(S; w) - f(T; w)}_+}_1 \geq C \cdot \das(S, T)^2, \\ \ \text{for all } S, T \in \Pvk
    \end{equation}
    We know that 
    \begin{equation*}
    \begin{split}
        &c \cdot \das(S, T)^2\leq \Exp_{w \sim \mu(\cdot)}\norm{\sbrac{f(S; w) - f(T; w)}_+}_1 \\
        \leq&\Prob(f(S;w)\geq f(T;w))\cdot \sup_{w\in \W} |f(S;w)\geq f(T;w)|
    \end{split}
    \end{equation*}
    Now, note that as $f$ is a bounded function, denote by $M$ the supremum of $f$ over all sets $S$ and weights $w\in \W$. 
    Thus by the triangle equality, 
    \begin{equation*}
            c \cdot \das(S, T)^2 \leq 2M\cdot \Prob(f(S;w)\geq f(T;w))
    \end{equation*}
    Thus, we obtain that
    \begin{align*}
        \frac{c \cdot \das(S, T)^2}{2\cdot M} \leq \Prob(f(S;w)\geq f(T;w))
    \end{align*}
    Thus, the complement satisfies:
    \begin{align*}
       \Prob(f(S;w) < f(T;w))\leq 1 - \frac{c \cdot \das(S, T)^2}{2\cdot M}
    \end{align*}
    Now, taking $m$ independeing copies of $f$, namely $F$, we have that 
    \begin{align*}
          \Prob(\forall i\in [m],F(S;w)_i < F(T;w)_i)\leq (1 - \frac{c \cdot \das(S, T)^2}{2\cdot M})^{m}
    \end{align*}
    Denoting by $C:=\frac{c}{2M}$ yield the desired:
    \begin{align*}
          \Prob(F(S;w) < F(T;w))\leq (1 - C \cdot \das(S, T)^2)^{m}
    \end{align*}
\end{proof}

\paragraph{Lower H\"older separability of $\relu$ networks} So far, we have only given results about Lower H\"older separability of \mas functions with Hat activations. However, as seen in Proposition \ref{prop:relu_mas}, we have seen that a two layer linear network with $\relu$ activation is weakly \mas, and as observed from the proof technique, such a 2 layer $\relu$ net can "simulate" a hat activation. In our arguments, we have shown Lower H\"older separability of Hat Activation based parametric set functions by finding a separating parameter point $w \in \Wcal$ for a given pair of sets $(S, T)$, and constructing an open set of non-zero measure, $\Wcal' \subseteq \Wcal$ around $w$. We lower bounded the expectation of the non-negative Random variable in the definition of Lower H\"older separability(in ~\cref{eq:lower_holder}) by restricting the expectation only to $\Wcal'$ and showing that it is proportional to $\das(S, T)^\lambda$. One can extend this exact same idea to construct the open set of parameters for the two later $\relu$ networks from the separable parameter point that we get from proof of Proposition \ref{prop:relu_mas}. 

Also, as obtained from Proposition \ref{prop:shallow_relu_not_sep}, we see that one layer $\relu$ networks are not even weakly \mas. But a key key step  in that proof uses that one can find 3 collinear elements in the ground set $V$ and construct sets with them to obtain the necessary counter-example. Thus, a natural question is to ask: whether shallow 1-layer networks with $\relu$ activation are also weakly \mas with some additional assumptions on the domain. We show in the following result that, with the additional assumption of $V$ being a hypersphere, we may guarantee Lower H\"older separability for even shallow $\relu$ networks.

\begin{theorem}
    \label{thm:shallow_relu_lower_holder}
    Given a ground set $V \subseteq \S^{d-1}$. We define $F: \Pvk \to \R$ as $F(S) = \sum_{x \in S}\relu\round{a\tran x + b}$, where $a \sim \Unif(\S^{d-1})$ and $b \sim \Unif(-1,1)$. Then, for any $S, T \in \Pvinf$, we have: $\Exp_{a, b}\sbrac{F(S) - F(T)}_+ \geq C(k) \das(S, T)^{(d+3)2^k}$ for some constant $C(k) > 0$ depending only on $k$.
\end{theorem}
\begin{proof}
    Let $S = \set{x_1, \cdots, x_M}$ and $T = \set{y_1, \cdots, y_N}$, where $M \leq N \leq k$. And according to the definition of $\das(S, T)$ from~\cref{eq:asymmetric_dist} we can equivalently write $\das(S, T) = \min_{\pi \in \S_N}\sum_{i=1}^M\norm{x_i - y_{\pi(i)}}_2$. Let $\pi^*$ be the optimal solution of the above problem, then we can define the map $\tau: [M] \to [N]$ as follows: $\tau(i) = \pi^*(i), \forall i \in [M]$. Essentially, $x_i$ and $y_{\tau(i)}$ are coupled in the optimal alignment. Now, for any given $i \in [M]$, let $\Delta_i := \norm{x_i - y_{\tau(i)}}_2$. Then, consider the open Euclidean ball $\B(x_i, \Delta_i)$. Then if $j \in [N]$ such that $y_j \in \B(x_i, \Delta_i)$ then, $j =\tau(\ell)$ for some $\ell \in [M]$. Otherwise switching from $i \mapsto \tau(i)$ to $i \mapsto j$ reduces the cost, which violates optimal alignment. Let $\Omega_i := \set{\ell \in [M]: y_{\tau(\ell)} \in \B(x_i, \Delta_i)}$. Here, we make the following claim:
    \begin{claim}
    \label{claim:stupid_claim}
    $\exists \alpha \in [M]$ such that $\sum_{j \in \Omega_i}\norm{y_j - x_{\tau^{-1}(j)}}_2 \leq \frac{\Delta_\alpha^2}{16}$ and $\Delta_\alpha \geq \frac{16\das^{\twok}(S, T)}{(16k)^{\twok}}$
    \end{claim}
    \begin{proof}
        Suppose not. WLOG, assume the following order $\Delta_1 \geq \Delta_2 \geq \cdots \geq \Delta_M$. We build the following directed graph $G$ on $[M]$. Start with the node $v_0 = 1$ and add one node at a time. At step $t$, we select $v_t := \arg\max_{v \in \Omega_{v_{t-1}}} \Delta_v$ and add the edge $(v_{t-1}, v_t)$. By the assumption, $\frac{\Delta_{v_t}}{16k} \geq (\frac{\Delta_{v_{t-1}}}{16k})^2 \geq \cdots \geq (\frac{\Delta_{v_0}}{(16k)})^{\twok}$. Since we work with finite $M$, either $G$ is a DAG and this process terminates or $G$ has a cycle. In former case at termination step $T$ we cannot add any more nodes, thus $\Omega_T = \emptyset$, and $T \leq k$. So, $\Delta_T \geq \frac{16k}{k}\frac{\das^{2^T}(S, T)}{(16k)^{2^T}} \geq  \frac{16\das^{\twok}(S, T)}{(16k)^{\twok}}$ and we're done. Otherwise, consider the smallest cycle in $G$. Let this cycle of length $p$ be $\mathcal{C} = \set{v_\beta \to v_{\beta+1} \to \cdots \to v_{\beta + {p-1}} \to v_\beta}$. Then, we have $\norm{x_{v_\beta + i} - y_{\tau(v_\beta + i +1)}}_2 < \norm{x_{v_\beta + i} - y_{\tau(v_\beta + i)}}_2, \forall i \in \set{0, \cdots, p-1}$ and addition of sub-indices is modulo $p$. Thus, we have the following: $\norm{x_{v_\beta} - y_{\tau(v_\beta+1)}}_2 + \norm{x_{v_\beta+1} - y_{\tau(v_\beta+2)}}_2 + \cdots + \norm{x_{v_\beta + p-1} - y_{\tau(v_\beta)}}_2 < \sum_{j=0}^{p-1}\norm{x_{v_\beta+i} - y_{\tau(v_\beta+j)}}_2$. Hence, switching from the coupling $i \mapsto \tau(i)$ to $i \mapsto \tau(i+1)$ for $i \in \set{v_\beta, \cdots, v_{\beta + p-2}}$ and to $v_\beta + p - 1 \mapsto v_\beta$ strictly reduces the cost, but we started with the optimal coupling. This gives a contradiction. 
    \end{proof}
    Consider $\alpha$ as in the above claim. Then we define the region: $\mathcal{A} := \set{a \in \S^{d-1} : \norm{a - x_\alpha} < \frac{\Delta_\alpha}4}$. Then, $\forall a \in \mathcal{A}, \norm{a - x_\alpha}_2 < \norm{a - y_{\tau(\alpha)}}_2$. Thus, $a \tran x_\alpha > a \tran y_{\tau(\alpha)}$. Also, consider the region $\mathcal{B}_a := \set{b \in \R : -1 + \frac{\Delta_\alpha^2}{16} + \frac{\Delta_\alpha^2}{64} + \frac{\norm{a-x}^2}2 \leq b \leq -1 + \frac{\Delta_\alpha^2}8}$. For any given $a \in \mathcal{A}$, we have: $\frac{\norm{a-x}^2}2 \leq \frac{\Delta_\alpha^2}{32}$, thus $-1 + \frac{\Delta_\alpha^2}{16}+ \frac{\Delta_\alpha^2}{64} + \frac{\norm{a-x}^2}2 \leq -1 + \frac{7}{64}\Delta_\alpha^2 \leq -1 + \frac{\Delta_\alpha^2}8$, so $\mathcal{B}_a \neq \emptyset$. Also note that, $\forall a \in \mathcal{A}, b \in \mathcal{B}_a$, we have:
    \begingroup
    \allowdisplaybreaks
    \begin{align*}
        &\sbrac{F(S) - F(T)}_+ \geq \sbrac{\relu\round{a\tran x_\alpha + b} - \round{\sum_{\beta \in \Omega_\alpha}\relu\round{a\tran y_{\tau(\beta)}} - \relu\round{a\tran x_{\beta}}}}_+ \\
        & \geq  \sbrac{\relu\round{a\tran x_\alpha + b} - \sum_{\beta \in \Omega_\alpha}\abs{\relu\round{a\tran y_{\tau(\beta)}} - \relu\round{a\tran x_{\beta}}}}_+\\
        & \geq \sbrac{\relu\round{a\tran x_\alpha + b} - \sum_{\beta \in \Omega_\alpha}\abs{a\tran(x-y)}}_+ \\
        & \geq \sbrac{\relu\round{a\tran x_\alpha + b} - \sum_{\beta \in \Omega_\alpha}\norm{x_\beta - y_{\tau(\beta)}}}_+ \\
        & \geq \sbrac{a\tran x_\alpha + b - \frac{\Delta_\alpha^2}{16}}_+ = \sbrac{1 + b - \frac{\norm{a-x}^2}2 - \frac{\Delta_\alpha^2}{16}}_+ \geq \frac{\Delta_\alpha^2}{64}
    \end{align*}
    \endgroup
    Thus, we have: 
    \begingroup
    \allowdisplaybreaks
    \begin{align*}
        &\Exp_{a, b}\sbrac{F(S) - F(T)}_+ \geq \Exp_{a,b}\sbrac{(F(S) - F(T))\indic{a \in \mathcal{A}, b \in \mathcal{B}_a}}_+\\
        &\geq \frac{\Delta_\alpha^2}{64} \Prob\curly{a \in \mathcal{A}, b \in \mathcal{B}_a} \\
        &= \frac{\Delta_\alpha^2}{64S_{d-1}} \int_{a \in \S^{d-1} : a_1 \geq 1 - \frac{\Delta_\alpha^2}{32}} \frac{1}{2}\round{\frac{3}{64}\Delta_\alpha^2 - (a_1 - 1)} da \\
        & \geq \frac{\Delta_\alpha^4}{2^{13}S_{d-1}} \int_{a \in \S^{d-1} : a_1 \geq 1 - \frac{\Delta_\alpha^2}{32}} da \\
        & = \frac{\Delta_\alpha^4}{2^{13}}\int_{1 - \frac{\Delta_\alpha^2}{32}}^1 (1-x^2)^{\frac{d-3}2} dx \geq \frac{\Delta_\alpha^6}{2^{18}}\round{\frac{\Delta_\alpha^2}{32}}^{\frac{d-3}2}\round{2 - \frac{\Delta_\alpha^2}{32}}^{\frac{d-3}2}\\
        &\geq \frac{\Delta_\alpha^{d+3}}{2^{23}}
    \end{align*}
    \endgroup
    Now, using the bounds from~\ref{claim:stupid_claim}, we get that, $\Delta_\alpha \geq C(k) \das^{\twok}(S,T)$. Thus we have that, $\Exp\sbrac{F(S) - F(T)}_+ \geq C_1(k)\das(S,T)^{(d+3)\twok}$, and done
\end{proof}

\paragraph{Upper Lipschitz property of weakly \mas functions} We now show that, the \mas functions obtained by applying one layer neural network with \hhat activation for each element of the set, followed by sum aggregation are also Upper Lipschitz continuous in expectation with respect to the augmented Wasserstein distance, following the framework of~\citet{davidson2025holderstabilitymultisetgraph}. This shows stability of such \mas embeddings. 
We say that a parametric set function $F: \Pvk \times \Wcal \to \R^m$ is Upper Lipschitz in expectation if $\exists C > 0$ such that:
\begin{equation}
    \Exp_{w\in \Wcal}  \norm{F(S;w)-F(T;w)}_1 \leq C \cdot \W^{(k)}(S, T), \ \forall S, T \in \Pvk
\end{equation}Where $\W^{(k)}(., *)$ is the augmented-Wasserstein metric for sets in $\Pvk$, where a padding $z$ is added for $k - |S|$ times to any multiset $S$ of size less than $k$. We work with a compact ground set $V$, and we pick a padding $z$ that has a positive distance from $V$. For the following proofs, we work with $V \subseteq \R^d$ which is norm bounded by $1$, and we choose the padding element $z \in \R^d$ such that $\norm{z} \geq 3$. For a general $V$ whose norm is bounded by $B$, we can scale the padding $z$ by $B$. With this, we show that, real valued functions $F: \Pvk \to \R$ given by $F(S) = \sum_{x \in S}\sigma\round{\frac{a\tran x - b}c}$ are Upper Lipschitz w.r.t $\W_k$. Now, if we take independent parameteric copies across $m$ output dimensions and if $\mtwo$ is a vector-to-vector Lipschitz function, then functions of the form $\mtwo \circ F$, where $F$ has the form in~\cref{eq:hadamard} are Upper Lipschitz as well.
 This we state in the following theorem:
 \begin{restatable}[$F$ is upper Lipschitz]{theorem}{gupperlipshiz}
     Let $\sigma$ belongs to the class of \hhat activations, and $V \subset \R^d$ be a compact ground set such that $\sup_{v \in V}\norm{v} \leq 1$. We consider $F: \Pvk \to \R, F(S) = \sum_{x \in S}\sigma\round{\frac{a\tran x - b}c}$, with the distributions $a \sim \Unif(\S^{d-1}), b \sim \Unif(-1,1), c \sim \Unif(0, 2)$. Then $\exists \text{ a constant }C > 0$, such that for any $S, T \in \Pvk$:
    \begin{align*}
        \Exp_{a, b, c} \abs{F(S) - F(T)} \leq C  \cdot \W^{(k)}(S,T)
    \end{align*}
\end{restatable}
\begin{proof}
Consider $S = \{x_1, \cdots, x_M\}, T = \{y_1, \cdots, y_N\}$. Since $\abs{F(S) - F(T)}$ is symmetric in $S, T$, we consider WLOG that $M \leq N$. As in previous sections, we use $\tau$ to denote an injective function from $[M] \to [N]$, and we use $\Omega_{N, M}$ to denote the set of all injective functions from $[M] \to [N]$. Thus, we can write the augmented Wasserstein distance on $\Pvk$ with padding $z$ as:
\begin{align*}
    \W^{(k)}(S, T) &= \arg\min_{\tau \in \Omega_{N, M}}\curly{\sum_{i=1}^M \norm{x_i - y_{\tau(i)}} + \sum_{j \neq \tau(i)} \norm{y_j - z}}
\end{align*} 
Let $\tau^*$ be the argmin of the above expression that gives the $\W_k(S, T)$ Now, we can write:
\begin{align*}
    & \Exp_{a, b,c}\abs{F(S) - F(T)} = \Exp_{a,b,c} \bigg|\sum_{i=1}^M \sigma\big(\frac{a\tran x_i - b}{c}\big) - \sigma\big(\frac{a\tran y_{\tau^*(i)} - b}{c}\big)  - \sum_{j\neq \tau^*(i)} \sigma\big(\frac{a\tran y_j - b}{c}\big)\bigg| \\
    \leq & \sum_{i=1}^M\Exp_{a, b,c}\abs{\sigma\round{\frac{a\tran x_i - b}{c}} - \sigma\round{\frac{a\tran y_{\tau^*(i)} - b}{c}}} + \sum_{j\neq \tau^*(i)}\Exp_{a, b,c} \abs{\sigma\round{\frac{a\tran y_j - b}{c}}}
\end{align*}

We now separately analyze coupled and un-coupled terms: i.e terms of the form $\Exp \bigg|\sigma\bigg(\frac{a\tran y_j - b}{c}\bigg)\bigg|$ and $\Exp\bigg|\sigma(\frac{a\tran x_i - b}{c}) - \sigma(\frac{a\tran y_{\tau^*(i)} - b}{c})\bigg|$ in the following analysis. For the proof that follows, we consider $\supp(\sigma) \subseteq [0,1]$. For support in $[\gamma_1, \gamma_2]$, the proof follows similarly by sifting the distributions of $b, c$ by $\gamma_1$ and scaling by $\gamma_2 - \gamma_1$.

\paragraph{Computing $\mathbf{\Exp \abs{\sigma\round{\frac{a\tran y_j - b}{c}}}}$:}

Note that, $\sigma(\frac{a\tran y_j - b}{c}) = 0, \forall b \geq a\tran y_j$. Also, for a given $b \in (-1, a\tran y_j)$ we have that:
\begingroup
\allowdisplaybreaks
\begin{align*}
    \forall c \in (0, a\tran y_j- b), \frac{a\tran y_j - b}{c} \geq 1 \implies \sigma\bigg(\frac{a\tran y_j - b}{c}\bigg) = 0
\end{align*}
\endgroup
Thus, conditioned on a given $a \in \S^{d-1}$, we can write the conditional expectation over $b, c$ as follows:
\begingroup
\allowdisplaybreaks
\begin{align*}
    &\Exp_{b,c} \curly{\sigma\big(\frac{a\tran y_j - b}{c}\big)\bigg| a}  \\
    = & \frac{1}{4}\int_{b=-1}^{a\tran y_j} \int_{c=a\tran y_j - b}^2 \bigg|\sigma\bigg(\frac{a\tran y_j - b}{c}\bigg) - \sigma(0)\bigg|\dif c \dif b,  \text{ using } \sigma(0) = 0\\
    \leq & \frac{\kappa_2}{4}\int_{b=-1}^{a\tran y_j} \int_{c=a\tran y_j - b}^2 \bigg(\frac{a\tran y_j - b}{c}\bigg) \dif c \dif b, \text{as }\sigma\text{ is a \hhat activation, it's Lipschitz}\\
    = & \frac{\kappa_2}{4}\int_{b=-1}^{a\tran y_j} (a\tran y_j - b)\ln\round{\frac{2}{a\tran y_j - b}} \dif b = \frac{\kappa_2}{4}\int_{a\tran y_j + 1}^0 z\ln\round{\frac 2z} (-\dif z), z:= a\tran y_j - b\\
    = &\frac{\kappa_2}{4}\int_0^{a\tran y_j + 1} z\ln\round{\frac 2z} \dif z = \frac{\kappa_2}{4}\cdot\frac{a\tran y_j + 1}2\round{\frac{a\tran y_j + 1}2 + \round{a\tran y_j + 1}\ln\round{\frac{2}{a\tran y_j + 1}}}\\
    \leq & \frac{\kappa_2}4 \cdot (a\tran y_j + a\tran a) = \frac{\kappa_2}4 \abs{a\tran(y_j + a)} \leq \frac{\kappa_2}4 \norm{y_j + a}\leq \frac{\kappa_2}4\norm{y_j - z}
\end{align*}
\endgroup
Now, taking an expectation over $a \sim \Unif(\S^{d-1})$, the inequality is preserved and we get that, for "isolated " $y_j$'s, $\Exp\sigma\round{\frac{a\tran - b}c} \leq \frac{\kappa_2}4\norm{y_j - z}$

Thus for "isolated" elements from the larger set that get coupled with the padded element $z$ while computing augmented-$\W^{(k)}$, we have shown that they are upper bounded by their diatcne from the padding $z$. Hence, it remains to show a similar Lipschitz upper bound for the elements $x_1, \cdots x_m \in S_1$ which have a corresponding $y_{i(1)}, \cdots, y_{i(m)} \in S_2$ in the optimal coupling\\

\paragraph{Computing $\mathbf{\Exp\bigg|\sigma(\frac{a\tran x_i - b}{c}) - \sigma(\frac{a\tran y_{\tau^*(i)} - b}{c})\bigg|}$: \\}

    Now, we have, for any $i \in [M]$, and a given $a \in \S^{d-1}$ the conditional expectation $\Exp_{b,c}\bigg|\sigma(\frac{a\tran x_i - b}{c}) - \sigma(\frac{a\tran y_{\tau^*(i)} - b}{c})\bigg|$ is symmetric in $x_i$ and $y_{\tau^*(i)}$. Thus, WLOG we can assume that, $a\tran x_i \geq a\tran y_{\tau^*(i)}$. Note that we then have:
    \begingroup
    \allowdisplaybreaks
    \begin{align*}
        &\bigg|\sigma(\frac{a\tran x_i - b}{c}) - \sigma(\frac{a\tran y_{\tau^*(i)} - b}{c})\bigg| \\
        =&\begin{cases}
            &\bigg|\sigma(\frac{a\tran x_i - b}{c})\bigg|, \forall b \in (a\tran y_{\tau^*(i)}, a\tran x_i), \forall c \in (a\tran x_i - b, 2)\\
            & \bigg|\sigma(\frac{a\tran x_i - b}{c}) - \sigma(\frac{a\tran y_{\tau^*(i)} - b}{c})\bigg|,  \forall b \in (-1, a\tran y_{\tau^*(i)}), \forall c \in (a\tran y_{\tau^*(i)} - b, 2)\\
            & 0, \text{ otherwise}
        \end{cases}
    \end{align*}
    \endgroup
Hence, conditioned on a specific $a \in \S^{d-1}$, we can write the conditional expectation over $b, c$ as:
\begingroup
\allowdisplaybreaks
\begin{align*}
    & \Exp_{b,c}\bigg|\sigma(\frac{a\tran x_i - b}{c}) - \sigma(\frac{a\tran y_{\tau^*(i)} - b}{c})\bigg| = \int_{b=a\tran y_{\tau^*(i)}}^{a\tran x_i} \int_{c = a\tran x_i - b}^2\sigma\round{\frac{a\tran x_i - b}c} \dif c \dif b\\
    & + \int_{b=-1}^{a\tran y_{\tau^*(i)}} \int_{c = a\tran y_{\tau^*(i)} - b}^2 \bigg|\sigma(\frac{a\tran x_i - b}{c}) - \sigma(\frac{a\tran y_{\tau^*(i)} - b}{c})\bigg| \dif c \dif b  \\
    \leq & \frac{\kappa_2}4\int_{b=a\tran y_{\tau^*(i)}}^{a\tran x_i} (a\tran x_i - b)\ln\round{\frac{2}{a\tran x_i - b}} \dif b \\
    &+\frac{\kappa_2}4 \abs{a\tran(x_i - y_{\tau^*(i)})}\int_{b=-1}^{a\tran y_{\tau^*(i)}} \ln\round{\frac{2}{a\tran y_{\tau^*(i)} - b}} \dif b  \\
    \leq & \frac{\kappa_2}4\round{\int_0^{a\tran(x_i - y_{\tau^*(i)})}z\ln\round{\frac 2z}\dif z + \abs{a\tran(x_i - y_{\tau^*(i)})}\int_0^{a\tran y_{\tau^*(i)}+1}\ln\round{\frac 2z}\dif z}\\
    \leq & \frac{\kappa_2}4\round{\abs{a\tran(x_i - y_{\tau^*(i)})}+2\abs{a\tran(x_i - y_{\tau^*(i)})}} < \kappa_2 \abs{a\tran(x_i - y_{\tau^*(i)})}\\
    \leq & \kappa_2 \norm{x_i - y_{\tau^*(i)}}, \text{ by Cauchy Schwarz}
\end{align*}
\endgroup
Thus taking one more expectation over $a \sim \S^{d-1}$, and summing over all $i \in [M]$, we get:
\begin{align*}
    \Exp_{a, b, c}\sum_{i=1}^M\bigg|\sigma(\frac{a\tran x_i-b}{c}) - \sigma(\frac{a\tran y_{\tau^*(i)}-b}{c})\bigg| \leq \kappa_2\sum_{i=1}^M \norm{x_i - y_{\tau^*(i)}}
\end{align*}
Thus, combining our computations for "isolated" $y_j$'s and "coupled" $(x_i, y_{\tau^*(i)})$'s we get that:
\begin{align*}
    \Exp_{a, b,c} |F(S) - F(T)| \leq \frac{\kappa_2}4 \sum_{j \neq \tau^*(i)}\norm{y_j - z}  + \kappa_2\sum_{i=1}^M \norm{x_i - y_{\tau^*(i)}} \leq \kappa_2 \W^{(k)}(S, T)\\
\end{align*}
And our proof of $F$ being upper Lipschitz is complete. 
\end{proof}

\subsection{Universal approximators with \mas functions}
\universal*
\begin{proof}
\label{universality}
A \mas function is in particular invertible, and therefore we can write $f(S)=f\circ F^{-1} \circ F(S) $. On the image of $F$ we can define the function $M=f\circ F^{-1}$ which is a vector-to-vector functions, and it is monotone on the image of $F$, because if $v=F(S), u=F(T)$ for some sets $S,T$, and if $v\leq u$, then by separability $S \subseteq T$, and therefore  
$$M(v)=f\circ F^{-1}(v)= f(S)\leq f(T)=f\circ F^{-1}(u)=M(u) $$

Finally, we need to show that $M$ can be extended to a montone function on all of $\RR^m$. we accomplish this by defining 
$$M(v)=\max\{M(u)| \quad u \leq v,  u \in \mathrm{Image}(F) \} .$$
\end{proof}

\section{Details about our model}
\label{app:model}
\subsection{Architecture details:}
In all 4 of \our models, we use an elementwise neural network $NN_\theta: \R^d \to \R^m$, where the ground set $V \subseteq \R^d$ and the output in in $\R^m$. We follow that with a sum aggregation, followed by a monotone vector-to-vector neural netwrok $M_{2, \phi}: \R^m \to \R^m$. For the embedding Neural Network, we use a 1 hidden layer NN with $\relu$ activation in the hidden layers and $\relu$ or \hhat activation in the output, depending on whether we're using \relumas or \hhat activation based \our. To make $M_{2, \phi}$ monotone, we use non-negative weights by taking the absolute value of parameters before applying the linear transformation, and use all monotonic activations like $\relu$ in the intermediate layers. Now, as discussed before, we give details of our model \integralmas.
\subsection{\integralmas}
In this appendix we design universal approximators for the class of Hat functions to learn $\sigma$. We parametrize For that, we derive equivalent conditions on the derivatives of $\sigma$ (per embedding dimension) and use the techniques of learning functions by modelling the derivatives using neural networks and then using a numerical integration as in \citet{wehenkel2021unconstrainedmonotonicneuralnetworks}. 
\begin{lemma}
\label{lemma:hat_derivative}
    If $\sigma: \R \to \R_{\geq 0}$ belongs to the hat activation class , then $\sigma\dash$ satisfies the following conditions for some positive constants $c, C > 0$ and $\alpha \in \R, \beta > 0, \gamma \in (0,1)$:
    \begin{enumerate}
    \item $\supp\round{\sigma\dash} \subseteq [\alpha, \alpha+\beta]$
    \item $\int_\alpha^{\alpha+\beta}\sigma\dash(x) \dif x = 0$
    \item $\sigma\dash(x) \leq C, \forall x \in \R$ 
    \item $\sigma\dash(x) \geq c, \forall x \in (\alpha, \alpha+\gamma \cdot \beta)$ 
    \end{enumerate}
The above 4 conditions, along with $\sigma(\alpha) = 0$ are the necessary-sufficienet conditions that characterizes the hat activation class.
\end{lemma}
\begin{proof}
    We first prove that, if $\sigma$ belongs to the hat activation class as defined in \ref{def:hhat}, then conditions 1-4 are satisfied. Note that, if $\sigma$ is a \hhat activation, then $\sigma$ has compact support in some $[\alpha, \alpha + \beta]$ by deifnition, and it's piecewise continuously differentiable. Thus, outside $[\alpha, \alpha + \beta]$, we have $\sigma \equiv 0 \implies \sigma\dash \equiv 0$. Hence, $\supp(\sigma\dash)\subseteq [\alpha, \alpha + \beta]$. We also have that, $0 = \sigma(\alpha) = \sigma(\alpha + \beta) \implies \sigma(\alpha+\beta) - \sigma(\alpha) = 0 \implies \int_\alpha^{\alpha+\beta}\sigma\dash(t) \dif t = 0$. Finally, $\sigma$ being Lipschitz is equivalent to $\abs{\sigma\dash(t)}\leq C$ for some $C > 0$ and by continuity of $\sigma\dash$ in $(\alpha, \alpha + \beta)$ we must have that $2c := \lim_{t\to\alpha^+}\sigma\dash(t) > 0 \implies \sigma\dash(t) \geq c, \forall t \in (\alpha , \alpha + \gamma \cdot \beta)$ for some $\gamma \in (0,1)$.  Thus, conditions 1-4 are implied by $\sigma$ belonging to the \hhat activation class. 
    
    Moreover, if conditions 1-4 are satisfied, and $\sigma(\alpha) = 0$, then condition-2 implies $\sigma(\alpha+\beta)=0$. This, along with condition-1 implies that $\sigma(t) = 0, \forall t \not \in [\alpha, \alpha + \beta]$. Condition-4 immediately implies that $\sigma(\alpha + \gamma \cdot \beta) \geq c \cdot \gamma\beta > 0$, thus $\sigma \not \equiv 0$. On the other hand, we get that $\sigma' \leq C$ implies by LMVT that, $\abs{\sigma(x) - \sigma(y)} \leq C \cdot \abs{x-y}$. Hence, $\sigma$ belongs to the class of \hhat activations following the definition from \ref{def:hhat}.
\end{proof}
\paragraph{Neural Parmetrization of \integralmas :}Consider $h_\theta^1(.), h_\phi^2(.)$ to be non-negative + bounded and bounded fully connected Neural Networks respectively of one hidden layer each. Also consider the trainable support parameters to be $\bm{\alpha} \in \R^m, \bm{\beta} \in \R^m_+, \bm{\gamma} \in (0,1)^m$ where $m$ is the output dimension and $\bm{\alpha}, \bm{\beta}, \bm{\gamma}$ are the support parameters, aggregated for all output dimensions.
Now, we define an integral based model of parametric hat functions.
Let $\Theta = (\theta, \phi, \alpha, \beta, m)$ be the parameter space.
\begin{equation}
\label{eq:neural_hatfn}
\begin{split}
    \sigma_\Theta(x) &= \int_\alpha^x h_\theta^1(z)\indic{\alpha \leq z\leq\alpha + \gamma\beta} \dif z - \sbrac{\frac{\int_\alpha^{\alpha+\gamma\beta}h_\theta^1(z)dz}{\int_{\alpha+\gamma\beta}^{\alpha + \beta}h_\phi^2(z)dz}}\round{\int_{\alpha + \gamma\beta}^{x}h_\phi^2(z)\indic{\alpha + \gamma\beta\leq z\leq\alpha + \beta}dz}   
\end{split}
\end{equation}
In the above formulation, in addition to the support parameters we also learn the function itself rather than performing a linear interpolation, this makes the above class a universal approximator of hat functions.
\begin{restatable}[Universal Hat approximator]{lemma}{lemma:universal_hat_approximator}
    The family $\sigma_\Theta$ defined in~\cref{eq:neural_hatfn} is an universal approximator of the class of \hhat functions having support in $[\alpha, \alpha + \beta]$ with $\sigma\dash(t) > 0, \forall t \in (\alpha, \alpha + \gamma \beta)$
\end{restatable}
\begin{proof}
    We show that the given parametric model is an universal approximator of \hhat functions using their equivalent formulation in Lemma~\ref{lemma:hat_derivative}. Let $\sigma_1$ be $\sigma$ restricted to the interval $[\alpha, \alpha + \gamma\beta]$ and $\sigma_2$ be $\sigma$ bestricted to the interval $[\alpha + \gamma\beta, \alpha + \beta]$. Using Theorem-1 from~\citet{lu2017expressivepowerneuralnetworks}, we can get $\relu$ neural network $h_\theta^1: \R \to \R$ and $h_\phi^2: \R \to \R$ of width 5 each such that approximates the derivatives of the restricted functions $\sigma_1\dash$ and $\sigma_2\dash$. Thus, the following hold for any $\epsilon > 0$ and any constant $c \in \R$:
    \begin{equation}
    \label{eq:univ_approx}
        \int_\alpha^{\alpha + \gamma\beta}\abs{\sigma_1\dash(z) - h_\theta^1(z)}\dif z \leq \epsilon \text{ and } \int_{\alpha+ \gamma\beta}^{\alpha + \beta}\abs{\sigma_2\dash(t) + c \cdot h_\phi^2(t)}\dif t \leq \epsilon
    \end{equation}
    We now observe that, if $x \in (-\infty, \alpha)$, both the indicator functions: $\indic{\alpha \leq z\leq\alpha + \gamma\beta}$ and $\indic{\alpha + \gamma\beta\leq z\leq\alpha + \beta}$ as in the integrands of ~\cref{eq:neural_hatfn} evaluate to 0, thus $\sigma_\Theta$ exactly coincides with $\sigma$. On the other hand, if $x \in (\alpha + \beta, \infty)$ then we would have the first integral of \ref{eq:neural_hatfn} evaluate to $\int_\alpha^{\alpha +\gamma\beta} h_\theta^1(z) \dif z$ (due to presence of the indicator $\indic{\alpha \leq z \leq \alpha +\gamma\beta}$ and the second expression evaluates to $- \sbrac{\frac{\int_\alpha^{\alpha+\gamma\beta}h_\theta^1(z)dz}{\int_{\alpha+\gamma\beta}^{\alpha + \beta}h_\phi^2(z)dz}}\round{\int_{\alpha + \gamma\beta}^{\alpha + \beta}h_\phi^2(z)dz} = -\int_\alpha^{\alpha+\gamma\beta}h_\theta^1(z)dz$. Thus, both the expressions sum to $0$ and $\sigma_\Theta$ coincides with $\sigma$ for $x \in (\alpha +\beta, \infty)$ as well. 
    
    Now we shall show that, for any $x \in (\alpha , \alpha + \beta)$ we must also have $|\sigma_\Theta(x) - \sigma(x)| \leq \epsilon$, implying convergence in $\sup$ norm. We consider the following cases:
    \begin{itemize}
        \item \textbf{Case-I:} When $x \in (\alpha, \alpha +\gamma\beta)$, we have that:
        \begingroup
        \allowdisplaybreaks
        \begin{align*}
            &\abs{\sigma(x) - \sigma_\Theta(x)} \\
            = & \abs{\int_\alpha^x \sigma_1\dash(z) \dif z - \int_\alpha^x h_\theta^1(z) \dif z}\\
            \leq & \int_\alpha^x \abs{\sigma_1\dash(z) -  h_\theta^1(z)} \dif z, \text{by traingle ineq.}\\
            \leq & \int_\alpha^{\alpha + \beta\gamma} \abs{\sigma_1\dash(z) -  h_\theta^1(z)} \dif z, \text{by non-negativity of integrand}\\
            \leq & \epsilon, \text{by ~\cref{eq:univ_approx}}
        \end{align*}
        \endgroup
        \item \textbf{Case-II:} When $x \in (\alpha + \gamma\beta, \alpha +\beta)$ we have:
        \begingroup
        \allowdisplaybreaks
        \begin{align*}
            &\abs{\sigma(x) - \sigma_\Theta(x)} \\
            = & \abs{\int_\alpha^x \sigma\dash(z) \dif z - \int_\alpha^{\alpha +\gamma\beta} h_\theta^1(z) \dif z - \int_{\alpha+\gamma\beta}^x h_\phi^2(z) \dif z}\\
            = & \abs{\int_\alpha^{\alpha+\gamma\beta}\round{\sigma_1\dash(z) - h_\theta^1(z)}\dif z +\int_{\alpha +\gamma\beta}^x\round{\sigma_2\dash(z) + \sbrac{\frac{\int_\alpha^{\alpha+\gamma\beta}h_\theta^1(z)dz}{\int_{\alpha+\gamma\beta}^{\alpha + \beta}h_\phi^2(z)dz}} h_\phi^2(z)}\dif z}\\
            \leq & \int_\alpha^{\alpha+\gamma\beta}\abs{\sigma_1\dash(z) - h_\theta^1(z)} \dif z + \int_{\alpha+\gamma\beta}^x \abs{\sigma_2\dash(z) + \sbrac{\frac{\int_\alpha^{\alpha+\gamma\beta}h_\theta^1(z)dz}{\int_{\alpha+\gamma\beta}^{\alpha + \beta}h_\phi^2(z)\dif z}} h_\phi^2(z)}\dif z, \text{tri. ineq.}\\
            \leq & \epsilon + \int_{\alpha+\gamma\beta}^{\alpha + \beta} \abs{\sigma_2\dash(z) + \sbrac{\frac{\int_\alpha^{\alpha+\gamma\beta}h_\theta^1(z)dz}{\int_{\alpha+\gamma\beta}^{\alpha + \beta}h_\phi^2(z)\dif z}} h_\phi^2(z)}\dif z, \text{by~\ref{eq:univ_approx} on 1st term \& tri. ineq. on 2nd}\\
            \leq & 2\epsilon
        \end{align*}
        \endgroup
        where last inequality is obtained  $c= \frac{\int_\alpha^{\alpha+\gamma\beta}h_\theta^1(z)dz}{\int_{\alpha+\gamma\beta}^{\alpha + \beta}h_\phi^2(z)\dif z}$ in the universal approximator for $\sigma_2\dash$ in ~\cref{eq:univ_approx}. Thus, just we see that, for any $\epsilon > 0$ and $\alpha \in \R, \beta > 0, \gamma \in (0, 1)$ we can find a universal approximator for the \hhat function class
    \end{itemize}
\end{proof}
\paragraph{Parametrizing $\alpha, \beta, \gamma$} We have provided an universal approximator for the \hhat function class given $\alpha, \beta, \gamma$ but we'd also want to have $\alpha, \beta, \gamma$ as learnable parameters. This is valid for both \hatmas and \integralmas, as in both cases we seek to learn the support parameters. Now, there is no constraint on $\alpha$, so we can directly initialize $\bm{\alpha} \in \R^{3m}$ and learn it through gradient descent. However, we have a constraint on $\beta$ that $\beta > 0$. For this, we first initialize $\bm{\beta_0}\in \R^m$ and we obtain $\bm{\beta}\in \R^m$ by appling co-ordinatewise positive transformation on $\bm{\beta_0}$. This we do in two ways, depending on which works better on a given task: \textbf{(I)} by simply taking $\bm{\beta_0}\mapsto \abs{\bm{\beta_0}}$, where the absolute value is taken pointwise. \textbf{(II)}by taking the pointwise transformation $\bm{\beta_0}\mapsto \text{ELU}\round{{\bm{\beta_0}}; \upsilon} +\upsilon$, where $\text{ELU}\round{\cdot; \upsilon}$ is the Exponential Linear Unit with hyperparameter $\upsilon > 0$. Note that, the ELU function is given by: $\text{ELU}\round{x; \upsilon} = \begin{cases}
    x , \text{if } x > 0\\
    \upsilon\round{e^x - 1}, \text{if }x \leq 0
\end{cases}$. For parametrizing $\gamma \in (0,1)$, we similarly initialize $\bm{\gamma_0}\in\R^m$ and then apply a pointwise transfprmation that takes each co-ordinate to $(0,1)$. We use tempered sigmoid function: $\bm{\gamma_0} \mapsto \text{Sigmoid}\round{\tau \cdot \bm{\gamma_0}}$ where $\tau > 0$ is a hyperparameter

\paragraph{Designing soft indicator functions:} In the integrands for \integralmas, we have some indicator functions of the form $\indic{a \leq \cdot \leq b}(x) = \begin{cases}
    1, \text{if } x \in [a, b]\\
    0, \text{o.w}
\end{cases}$. We want these indicator functions to have non zero gradients as we intend to learn the "support parameters" $a, b$ of such indicator function. If we chose binary indicators, then that is not differentiable everywhere and have gradient $0$ in most places, making it difficult to learn $a, b$. Thus, we design a "soft" indicator function with the help of tempered sigmoid with hyperparameter $\tau_S$ as follows: 
\begin{align*}
    \indic{a \leq \cdot \leq b}(x) = \text{Sigmoid}\round{\tau_S \cdot (x - a)}\cdot \text{Sigmoid}\round{\tau_S \cdot (b-x)}
\end{align*}

\subsection{Making \our Mini Batch Consistent}
There has been a chain of works to make set-based neural models to be mini batch consistent in~\cite{MBC1_slotset, MBC2} which have applications in large scale machine learning. The goal of these Mini Batch Consistent (MBC) works is to design a set function $F$, which, if applied on any subsets from partitions of a set $S$ and pooled through an activation $g$, should return $F(S)$. Thus, at a high level, MBC functions analyse and design functions for the following batch consistency condition: $g(F(S_1), \cdots, F(S_n)) = F(S)$. MBC methods achieve this by imposing certain restrictions on attention-based set architectures. Similar to this, our work also puts restrictions on the inner embedding transformation and the outer vector-to-vector transformations. It is thus a natural question to ask, if \our can be made minibatch consistent. Since set containment is the major theme of our work, with potential applications in retrieval and recommendation systems, processing huge sets at once can be a key bottleneck, and it would be nice to make \our process over mini-batches and then design some aggregation mechanism to make the aggregated embedding equal to the embedding of the entire set.

If we consider a partition of a given set $S = \bigcup_{i=1}^k S_i$, where $S_i$'s are all disjoint, then $\our (S_i) \leq \our(S)$, by monotonicity. But, we can try to find an appropriate pooling mechanism over the subset embeddings to make the aggregated embedding same as that of the entire set. For example, if the outer transformation $M_{\theta_2}$, as in~\cref{eq:masnet} was invertible, and then we may define a particular permutation-invariant pooling function $g$ such that, $g(X_1, \cdots, X_k) = M_{\theta_2}\round{\sum_{i=1}^kM_{\theta_2}^{-1}\round{X_i}}$, then we'd have: $g\round{\our(S_1), \cdots, \our(S_k)} = M_{\theta_2}\round{M_{\theta_2}^{-1}\round{\our(S_i)}} = M_{\theta_2}\round{\sum_{i=1}^k\sum_{x \in S_i}\sigma \circ M_{\theta_1}(x)} = M_{\theta_2}\round{\sum_{x \in S} \sigma \circ M_{\theta_1}(x)} = \our (S)$. his is irrespective of how we partition $S$ into $S_1, \cdots S_k$. And this would enable us to process the subsets independently and aggregate them accordingly. 

For our purposes, we have needed the outer transformation $M_{\theta_2}$ to be a vector-to-vector monotone function in~\cref{eq:masnet}. But that doesn't guarantee injectivity or invertibility. Thus, we need to restrict the choice of $M_{\theta_2}$ to be monotone and invertible for \our to be MBC consistent. One possible choice is to use ideas from relevant works like~\cite{InvertibleNN, wehenkel2021unconstrainedmonotonicneuralnetworks} to design monotone and invertible neural networks, and analyse the expressibility properties of the resulting set model. This forms an interesting direction for future work. 

\section{Additional details about experiments and more experiments}
\label{app:addl-details}
\subsection{Hardware details:} All experiments were performed in a compute server, running the OS of GNU Linux Version 12, equipped with a 16 core Intel(R) Xeon(R) Gold 6130 CPU @ 2.10GHz CPU architecture and equipped with a cluster of 6 NVIDIA RTX A6000 GPUs with a memory of 49GB each.

\subsection{Dataset details and more experiments}
We now give an account of preparation details of each dataset and accounts on more experiments on each.

\paragraph{Synthetic datasets and related experiments: }

We begin with a controlled synthetic setting where we generate pairs $(S, T)$ such that $k_S < k_T$. 

Each pair is labeled either positive or negative. For positive examples, we first sample $k_T$ vectors in $\R^d$ from $\Normal(0, \Id)$, which form the set $T$. A subset of size $k_S$ is then uniformly sampled from $T$ to obtain $S$. For negative examples, we again sample $T$ as above, but then independently draw $k_S$ vectors from $\R^d$ to form $S$ .

To simulate a real world scenerio of information loss, we inject gaussian noise into the vectors of $S$ to get $S'$ but keep the boolean label unchanged.
\begin{figure}[h]
    \centering
    \begin{minipage}{0.48\textwidth}
        \centering
        \includegraphics[width=\textwidth]{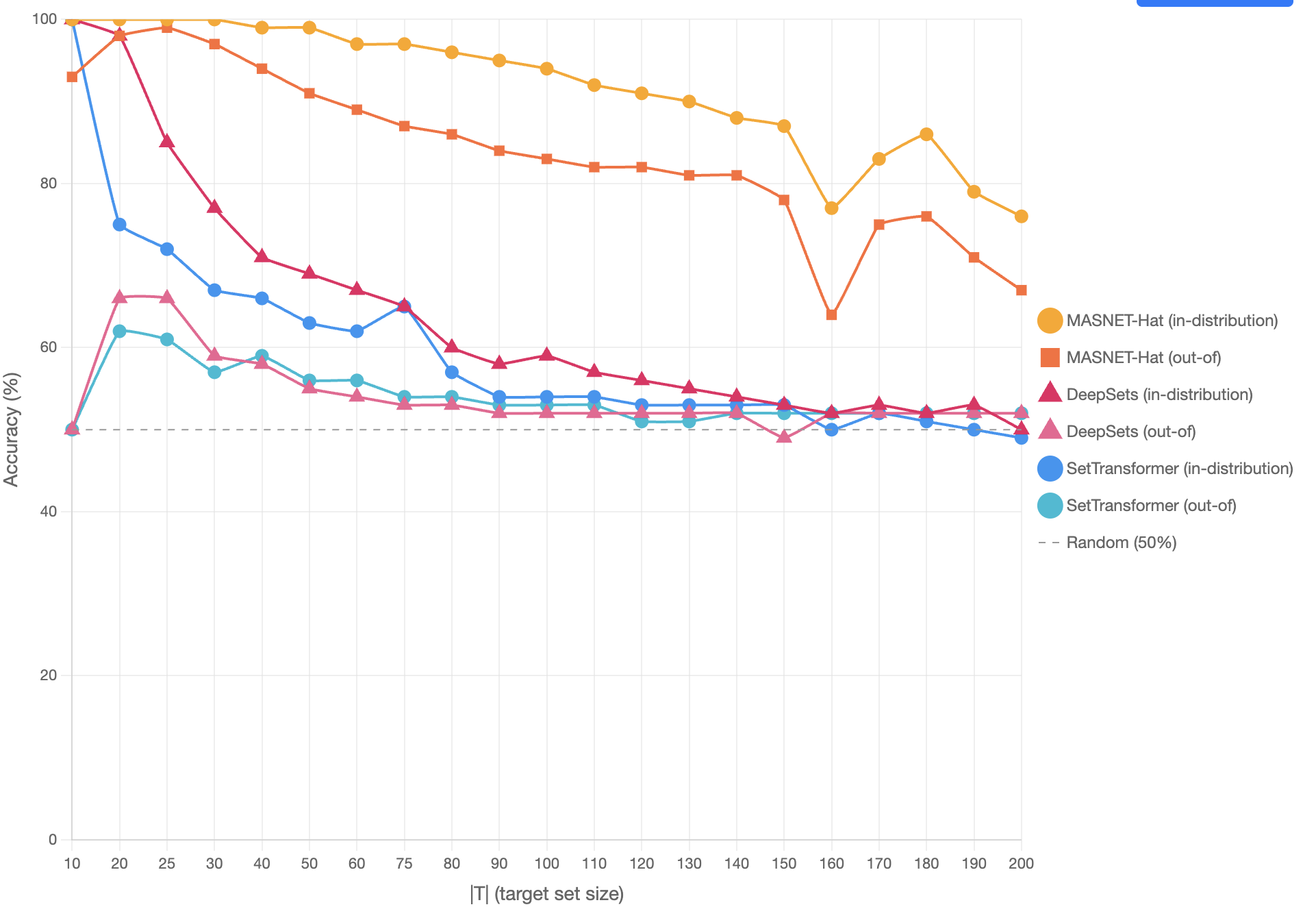}
        \caption{Acc vs $|T|$ for all models $\geq 2$ layer MLP}
        \label{fig:deep1}
    \end{minipage}
    \hfill
    \begin{minipage}{0.48\textwidth}
        \centering
        \includegraphics[width=\textwidth]{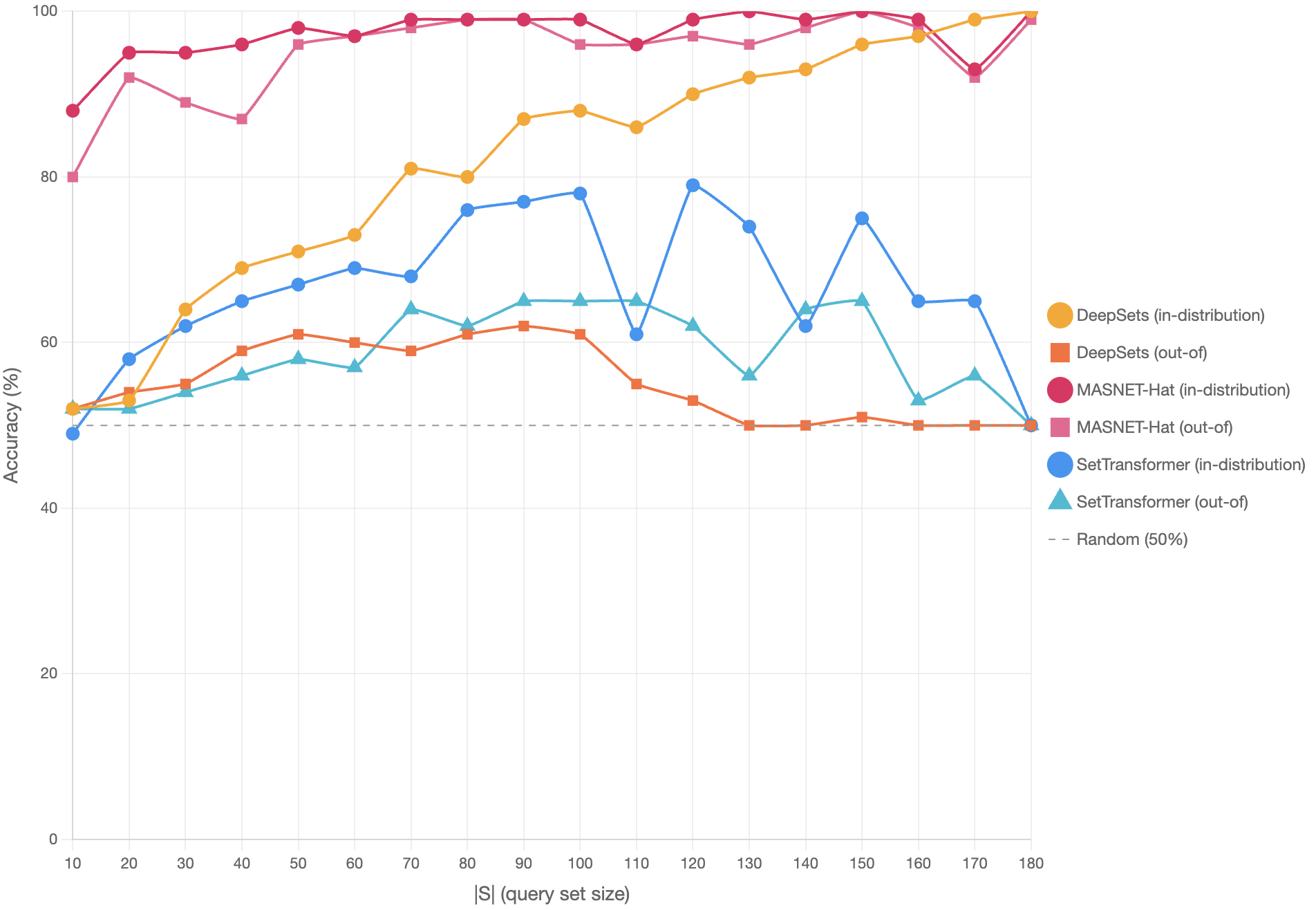}
        \caption{Acc vs $|S|$ for all models $\geq 2$ layer MLP}
        \label{fig:deep2}
    \end{minipage}
\end{figure}

On this, we provide two more plots: Figs. \ref{fig:deep1} and \ref{fig:deep2} that compares the accuracies of $\our$ vs baselines, by varying the gap between target set size $|T|$ and query set size $|S|$. Note that the more this gap is, the more it becomes challenging to separate two non-subsets. In all cases, the embedding MLP $M_{\theta_1}$ from~\cref{eq:masnet} is comprising of $\geq 2$ layers.

\paragraph{Text-based datasets and related experiments:}

We evaluate our models on three text based real-world datasets: MSWEB, MSNBC, and the Amazon Baby Registry datasets, all of which reflect naturally occurring set containment structures coming out of user behavior in websites and recommendation engines.

\paragraph{MSWEB and MSNBC:} These datasets consist of user activity logs from \textit{www.microsoft.com} and \textit{www.msnbc.com}, respectively. Each user session is treated as a bag of page identifiers, which are embedded into 768-dimensional vectors using a pre-trained BERT model. We construct query-target pairs \((S, T)\) by sampling \(S\) from a session \(V\), and letting \(T = V \setminus S\). A pair is labeled \(1\) if \(S \subseteq T\), and \(0\) otherwise. To simulate noisy real-world conditions, Gaussian noise is added to each element in \(S\) while preserving the label. The positive-to-negative ratio was taken to be 0.1, which is the number of labels 1 to labels 0 in the entire dataset.

\paragraph{Amazon Baby Registry:} This dataset contains subsets of products selected by customers, each tagged by category (e.g., “toys” or “feeding”). Product descriptions are embedded using BERT. We filter out sets of size \(<2\) or \(>30\), and for each valid subset \(S\), we generate \(T \supseteq S\) of size 30 by sampling additional items from the same category. For negative samples, \(T\) is sampled randomly and verified to satisfy \(S \not\subseteq T\).

\begin{wraptable}[7]{r}{0.6\textwidth}
\vspace{-15pt}
\centering
\small
\maxsizebox{\hsize}{!}{\tabcolsep 1pt 
\begin{tabular}{l|c|c|c|c}
\hline
\textbf{Model}  & \textbf{Bedding} & \textbf{Feeding} & \textbf{MSWEB} & \textbf{MSNBC} \\
\hline
DeepSets        & \underline{0.91}           & 0.90          & 0.93          & \textbf{0.97}           \\
SetTransformer  & \underline{0.91}           & 0.90          & 0.94          & \underline{0.96}         \\
FlexSubNet      & 0.90           & 0.91          & 0.94          & 0.94              \\
Neural SFE      & 0.91           & 0.91          & 0.91          & 0.93              \\
\hline
\relumas        & \textbf{0.95}  & \textbf{0.93} & \textbf{0.97} & \textbf{0.97}          \\
\hatmas         & \underline{0.91}  & \underline{0.92} & \underline{0.95} & \textbf{0.97}           \\
\hline
\end{tabular}}
\caption{Set containment in text datasets with 90:10 ratio}
\label{tab:noisy_accuracy_9010}
\end{wraptable}

Noise is added to \(S\) in both cases. This models scenarios such as predicting whether a set of displayed products \(T\) includes a customer’s interest \(S\). On these datasets, we provide three sets of additional experiments. In the first table, we perform the same set-containment task as done in the main text, but with a different negative-to-positive ratio of $90:10$. 

Note that, due to the inductive bias of monotonicity, all the positive  examples are correctly classified by design is case of \our. Thus, it only has to learn to identify the to separate the negative examples. Thus, if the test set has a higher proportion of negative examples, then it throws a model a toughter task to learn. This is shown in the following Table~\ref{tab:noisy_accuracy_9010}. Other than the modified class ratio, we provide two additional tables: in the first one, we compare shallow(1 layer) vs deep ($\geq 2$ layers) embedding MLP $M_{\theta_1}$ in \our , which is shown in Table~\ref{tab:shallow_deep}. White noise of std 0.1 added for \inexact set containment.

\begin{wraptable}{r}{0.6\textwidth}
\centering
\scriptsize
\renewcommand{\arraystretch}{1.1}
\setlength{\tabcolsep}{3pt}
\begin{tabular}{l|c|c|c|c|}
    \toprule
    \textbf{Model} & \textbf{bedding} & \textbf{feeding} & \textbf{MSWEB} & \textbf{MSNBC} \\
    \midrule
    Shallow-ReLU & 0.69 & 0.36 & 0.92 & 0.93 \\
    Shallow-\trimas & 0.83 & 0.76 & 0.94 & 0.95 \\
    Shallow-\hatmas & 0.89 & 0.77 & \emph{0.97} & 0.96 \\
    Shallow-\integralmas & 0.90 & 0.91 & 0.92 & 0.95 \\
    \hline
    (Deep)\relumas & \textbf{0.95} & \textbf{0.93} & \emph{0.97} & \textbf{0.97} \\
    Deep-\trimas & 0.92 & \emph{0.93} & \textbf{0.98} & 0.96 \\
    Deep-\hatmas & 0.91 & 0.92 & 0.95 & \emph{0.97} \\
    Deep-\integralmas & \emph{0.93} & 0.92 & 0.96 & \emph{0.97} \\
    \bottomrule
\end{tabular}
\caption{Comparison of shallow vs. deep variants of \our across datasets.}
\label{tab:shallow_deep}
\end{wraptable}

\paragraph{Ablation studies on text datasets:}
From the definition of \our in~\cref{eq:masnet}, we see that, the formulation os quite similar to Deepsets, which is of the form $F(S) = \mtwo\round{\sum_{x \in S}\relu(a\tran\mone(x) + b)}$. As shown earlier, using $\relu$ in the last layer of an elementwise function with a universally approximating $\mone$ is an instance of \our. 

\begin{wraptable}{r}{0.6\textwidth}
\centering
\scriptsize
\renewcommand{\arraystretch}{1.1}
\setlength{\tabcolsep}{3pt}
\begin{tabular}{l|c|c|c|c|}
    \toprule
    \textbf{Model} & \textbf{bedding} & \textbf{feeding} & \textbf{MSWEB} & \textbf{MSNBC} \\
    \midrule
    DeepSets & \textit{0.91} & 0.90 & 0.93 & \textit{0.97} \\
    DeepSets, monotone $\mtwo$ & 0.90 & 0.91 & 0.94 & 0.93 \\
    \relumas & \textbf{0.95} & \textbf{0.93} & \textbf{0.97} & \textbf{0.97} \\
    \hatmas & 0.91 & \textit{0.92} & \textit{0.95} & \textit{0.97} \\
    \hatmas (No division) & 0.87 & 0.92 & 0.92 & 0.94 \\
    \bottomrule
\end{tabular}
\caption{Ablation study }
\label{tab:ablation_study}
\end{wraptable}

But it is different from DeepSets in a key points, namely: $\bullet$ For set containment tasks, we don't have an outer $\mtwo$, which DeepSets have; for universal approximation tasks, we use a monotonially increasing $\mtwo$ that we enforce by taking positive weights and increasing activation functions. Also, for \hhat function based \our models, $\bullet$ we use a re-parametrization in which we perform a division based scaling. We now show the effect of the outer monotonic $\mtwo$ on DeepSets and division based re-parametrization on \hatmas in the  ablation study in table~\ref{tab:ablation_study}. White noise of std 0.1 added for \inexact containment.

\paragraph{Pointcloud datasets and related experiments}

ModelNet40 \cite{modelnetwu20153dshapenetsdeeprepresentation} is a benchmark dataset of 12,311 CAD models across 40 object categories, with each object represented as a 3D point cloud. We frame our task as checking if a given pointcloud $S$ is a segmenet of a target pointcloud $T$.

\begin{wraptable}[8]{r}{0.40\textwidth}
\vspace{-10pt}
\centering
\maxsizebox{\hsize}{!}{\tabcolsep 10pt
\begin{tabular}{l|c|c|c}
\hline
$|S|\rightarrow$ & $ 128$ & $ 256$ & $ 512$ \\
\hline
DeepSets        & \underline{0.90} & 0.90 & 0.91 \\
SetTransformer  & \underline{0.90} & \underline{0.91} & 0.91 \\
FlexSubNet      & 0.84 & 0.88 & 0.92 \\
Neural SFE      & 0.89 & 0.89 & 0.90 \\
\hline
\relumas        & \textbf{0.91}    & \textbf{0.93}    & \textbf{0.98} \\
\hatmas         & 0.87 & \underline{0.91} & \underline{0.94} \\
\hline
\end{tabular}}
\captionsetup{font=small}
\caption{Performance on Point cloud for different values of $|S|$, for 90:10 class ratio}
\label{tab:pointcloud_9010}
\end{wraptable}

Firstly, we choose an object from an object category $C_1$, and randomly sample $1024$ points from that object to get the target point cloud $T$. To obtain a positive sample (i.e true subset) from $T$, we first sample a random center point from $T$, and then extract $S$ using a hybrid approach: selecting the nearest point to the center, a few local neighbors via k-NN, and the rest via importance-weighted sampling (inverse-distance from center with noise). This makes sure that we're selecting a true local region from the point cloud, which is in-line with the actual task of detecting whether a given segment of an object is contained in a target object. For a negative sample(non-subset), we sample $S$ from an object of a different category $C_2$.

\begin{wraptable}[8]{r}{0.40\textwidth}
\vspace{-3pt}
\centering
\maxsizebox{\hsize}{!}{\tabcolsep 10pt
\begin{tabular}{l|c|c|c}
\hline
$|S|\rightarrow$ & $ 128$ & $ 256$ & $ 512$ \\
\hline
DeepSets        & \underline{0.89} & 0.90 & 0.90 \\
SetTransformer  & \underline{0.90} & \underline{0.91} & 0.91 \\
\hline
\relumas        & 0.87    & \textbf{0.91}    & \textbf{0.93} \\
\hatmas         & 0.85 & \underline{0.91} & 0.89 \\
\trimas        & 0.83    & 0.93    & 0.87 \\
\integralmas   & 0.87 & 0.90 & \underline{0.92} \\
\hline
\end{tabular}}
\captionsetup{font=small}
\caption{Point cloud with DGCNN for different values of $|S|$.}
\label{tab:pointcloud_dgcnn}
\end{wraptable}

We first give the accuracy table for PointNet encoder, but with the modified negative-to-positive class ratio of 90:10 to make the task harder for \our, which correctly classifies the actual subsets by the induictive bias of monotonicity. The numbers with the modified class-ratio are given in the following table of Tab~\cref{tab:pointcloud_9010}:

We also give the accuracy numbers for DGCNN encoder, in Table~\ref{tab:pointcloud_dgcnn}

\paragraph{Details on Linear assignment problem}
In this problem, we are given a positive matrix $M\in \R^{n \times m},n\leq m$, where $M_{i,j} $ 
represents the salary that a 'worker' will be paid to do a 'job' $j$. The goal of the task is to maximize the average salary obtained by all workers. This is done by finding the optimal assignment $\pi \in S_{n,m} $ which maps a worker $i\in [n]$ to a 'job' $\pi(i)\in [m]$, where by construction each worker can be mapped to at most one job, so $\pi(i)\in \{0,1\}$. This gives us the following maximization problem:
\begin{align*}
  \textstyle F(M)= \frac{1}{n}\max_{\pi \in S_{n,m}}\sum^{n}_{i=1} M_{i,\pi(i)}.
\end{align*}
Thinking of the matrix $M$ as a set of $m$ columns $M=[M_1,\ldots,M_m] $, we see that the function $F$ is permutation invariant and monotone. Accordingly, our goal will be to evaluate our \our model and baselines.

\end{document}